\documentclass[11pt,letterpaper]{article}

\title{\bf Online Improper Learning with an Approximation Oracle}
\date{}

\usepackage{times}
\usepackage[utf8]{inputenc}
\usepackage{float,framed}
\usepackage{algorithm}
\usepackage{algorithmic}
\usepackage{enumerate}
\usepackage{bm}
\usepackage{xspace}
\usepackage{natbib}
\usepackage{url}
\usepackage{hyperref}
\hypersetup{colorlinks,
	linkcolor=blue,
	citecolor=blue,
	urlcolor=magenta,
	linktocpage,
	plainpages=false}
\usepackage{pgf,tikz}
\usepackage{mathrsfs}
\usepackage{subcaption}
\usetikzlibrary{arrows}
\usepackage[
letterpaper,
top=1in,
bottom=1in,
left=1in,
right=1in]{geometry}

\usepackage{amsmath,amsfonts,amssymb}
\usepackage{mathtools}
\usepackage{amsthm} % better to load after ams math
\usepackage{latexsym}
\usepackage{relsize}

\usepackage[capitalise]{cleveref}
\usepackage{xcolor}
\usepackage{dsfont}

\newcommand{\A}{\mathcal{A}}

\newcommand{\C}{\mathcal{C}}

\renewcommand{\S}{\mathcal{S}}

\newcommand{\dx}{\mathrm{ d}x}
\newcommand{\dw}{\mathrm{ d}w}
\newcommand{\dy}{\mathrm{ d}y}
\newcommand{\vol}[1]{\mathrm{Vol}(#1)}
\newcommand{\ch}{\mathrm{CH}}

\newcommand{\interior}{\mathrm{int}}
\newcommand{\Reg}{\mathrm{Reg}}

\newcommand{\K}{\ensuremath{\mathcal K}}

\newcommand{\ignore}[1]{}

\newcommand{\email}[1]{\texttt{#1}}

\theoremstyle{plain}
 \newtheorem{theorem}{Theorem}[section]
 \newtheorem{lemma}[theorem]{Lemma}
 \newtheorem{corollary}[theorem]{Corollary}
  \newtheorem{definition}[theorem]{Definition}

\newtheorem*{remark*}{Remark}
\DeclareMathAlphabet{\mathbfsf}{\encodingdefault}{\sfdefault}{bx}{n}

\DeclareMathOperator*{\argmin}{arg\,min}

\DeclareMathOperator*{\poly}{poly}

\newcommand{\norm}[1]{\|#1\|}

\renewcommand{\O}{\mathcal{O}}

\newcommand{\E}{\mathbb{E}}

\newcommand{\pad}{\mathcal{PAD}}

\newcommand{\reals}{\mathbb{R}}

\renewcommand{\leq}{~\le~}
\renewcommand{\geq}{~\ge~}

\let\oldtfrac\tfrac
\renewcommand{\tfrac}[2]{\smash{\oldtfrac{#1}{#2}}}

\let\nablaold\nabla
\renewcommand{\nabla}{\nablaold\mkern-2.5mu}

\allowdisplaybreaks

\author{
	Elad Hazan\footnote{Princeton University and Google Brain. Email: \texttt{ehazan@cs.princeton.edu}}
	\and
	Wei Hu\footnote{Princeton University. Email: \texttt{huwei@cs.princeton.edu}}
	\and
	Yuanzhi Li\footnote{Princeton University. Email: \texttt{yuanzhil@cs.princeton.edu}}
	\and
	Zhiyuan Li\footnote{Princeton University. Email: \texttt{zhiyuanli@cs.princeton.edu}}
}

\begin{document}

\maketitle

\begin{abstract}
We revisit the question of reducing online learning to approximate optimization of the offline problem. In this setting, we give two algorithms with near-optimal performance in the full information setting:  they guarantee optimal regret and require only poly-logarithmically many calls to the approximation oracle per iteration. Furthermore, these algorithms apply to the more general improper learning problems. % in which losses can be real-valued, as opposed to being non-negative in previous work.
In the bandit setting, our algorithm also significantly improves the best previously known oracle complexity while maintaining the same regret.
\end{abstract}

%\begin{keywords}
%Online learning, approximation algorithms, regret minimization, bandit algorithms
%\end{keywords}

\section{Introduction}

%The rise of non-convex models, especially deep neural networks, as the state-of-the-art in machine learning stress the importance of studying the {\it computational} relationship between optimization and learning.  

One of the most fundamental and well-studied questions in learning theory is whether one can learn a given problem using an optimization oracle.  For online learning in games, it was shown by \cite{kalai2005efficient} that an optimization oracle giving the best decision in hindsight is sufficient for attaining optimal regret. 

However, in many non-convex settings, such an optimization oracle is either unavailable or NP-hard to compute. In contrast, in many such cases, efficient approximation algorithms are usually known, and are guaranteed to return a solution within a certain multiplicative factor of the optimum. These include not only combinatorial optimization problems such as {\sc Max Cut, Weighted Set Cover, Metric Traveling Salesman Problem, Set Packing}, etc., but also machine learning problems such as {\sc Low Rank Matrix Completion}. 

\cite{kakade2009playing} considered the question of whether an approximation algorithm is sufficient to obtain vanishing regret compared with an approximation to the best solution in hindsight. They gave an algorithm for this offline-to-online conversion. However, their reduction is inefficient in the number of per-iteration queries to the approximation oracle, which grows linearly with time. Ideally, an efficient reduction should call the oracle only a constant number of times per iteration and guarantee optimal regret at the same time, and this was considered an open question in the literature.  

Various authors have improved upon this original offline-to-online reduction under certain cases, as we survey below. 
Recently, \cite{garber2017efficient} has made significant progress by giving a more efficient reduction, which improves the number of oracle calls both in the full information and the bandit settings. He explicitly asked whether a near-optimal reduction with only logarithmically many calls per iteration exists.

\subsection{Our Results}

In this paper we resolve this question on the positive side, and in a more general setting. We give two different algorithms in the full information setting, one based on the \emph{online mirror descent (OMD)} method and another based on the \emph{continuous multiplicative weight update (CMWU)} algorithm, which give optimal regret and are oracle-efficient.  Furthermore, our algorithms apply to more general loss vectors. Our results are summarized in the table below. 

\begin{table}[htp]
\caption{Summary of results in the full information setting.}
\begin{center}
\begin{tabular}{|c|c|c|c|}
\hline
Algorithm & $\alpha$-regret in $T$  rounds & oracle complexity per round & loss vectors  \\
\hline
\hline
\cite{kakade2009playing}  & $O(\sqrt{T})$  & $O(T)$  & general  \\
\hline
\cite{garber2017efficient} & $O(\sqrt{T})$  & $O(\sqrt{T} \log T  )$ & non-negative  \\
\hline
Algorithm~\ref{alg:OMD}  (OMD) &  $O(\sqrt{T})$  & $O(\log T  )$ & PNIP property  \\
\hline
Algorithm~\ref{alg:improper_CMWU} (CMWU) & $\widetilde O(\sqrt{T})$  & $O(\log T)$ &  general \\
\hline
\end{tabular}
\end{center}
\label{default}
\end{table}%

In addition to these two algorithms, we give an improved bandit algorithm based on OMD: it attains the same $O(T^{2/3})$ regret as in \citep{kakade2009playing,garber2017efficient} with a lower computational cost: our method requires $\widetilde O(T^{2/3})$ oracle calls over all the $T$ game iterations, as opposed to $\widetilde O(T)$ in the previous best method.

Besides the improved oracle complexity, our methods have the following additional advantages:
\begin{itemize}
\item
While the algorithm in \citep{garber2017efficient} requires non-negative loss vectors, our second algorithm, based on CMWU, can work with general loss vectors. Furthermore, our OMD-based algorithm can also work with loss vectors from any convex cone satisfying the \emph{pairwise non-negative inner product (PNIP)} property defined in Definition~\ref{def:pnip} (together with an appropriately chosen regularizer), which is more general than the non-negative orthant. 

\item
Our methods apply to a general online improper learning setting, in which the predictions can be from a potentially different set from the target set to compete against. Previous work considered this different set to be a constant multiple of the target set, which makes sense primarily for combinatorial optimization problems.

However, in many interesting problems, such as {\sc Low Rank Matrix Completion}, the natural approximation algorithm returns a matrix of higher rank. This is not in a constant multiple of the set of all low rank matrices, and our additional generality allows us to obtain meaningful results even for this case. 

\item
Our first algorithm is based on the general OMD methodology, and thus allows any strongly convex regularizer. This can give better regret bounds, in terms of the space geometry, compared with the previous algorithm of \citep{garber2017efficient} that is based on online gradient descent and Euclidean regularization. The improvement in regret bounds can be as large as the dimension. 
\item
Our bandit algorithm is based on OMD with a new regularizer that is inspired from the construction of barycentric spanners, and may be of independent interest. 
\end{itemize}

\subsection{Our Techniques} 

The more general one of our algorithms is based on a completely different methodology compared with previous online-to-offline reductions. It is a variant of the \emph{continuous multiplicative weight update (CMWU)} algorithm, or the \emph{continuous hedge} algorithm. Our idea is to apply CMWU over a superset of the target set, and in every iteration the algorithm tries to play the mean of a log-linear distribution. To check feasibility of this mean,  we show how to design a \emph{separation-or-decomposition oracle}, which either certifies that the mean is infeasible - in this case it provides a separating hyperplane between the mean and the target set and thus gives a more refined superset of the target set, or provides a distribution over feasible points whose average is superior to the mean in terms of the regret.
Using this approach, the more oracle calls the algorithm makes, the tighter superset it can obtain, and we show an interesting trade-off between the oracle complexity and the regret bound.
%, and construct distributions over ever diminishing sets (that all contain the decision set). 
%The implementation of this algorithm requires sampling from log-concave distributions, as well as cutting plane methods and multiplicative-weights refinements. 

The other algorithm follows the line of \cite{garber2017efficient}. We show how to significantly speed up Garber's \emph{infeasible projection oracle}, and to generalize Garber's algorithm from \emph{online gradient descent (OGD)} to \emph{online mirror descent (OMD)}.

This additional generality is crucial in our bandit algorithm, where we make use of a novel regularizer in OMD, called the \emph{barycentric regularizer}, in order to have a low-variance unbiased estimator of the loss vector. This geometric regularization may be of independent interest.

\subsection{Related Work}

The reduction from online learning to offline approximation algorithms was already considered by \cite{kalai2005efficient}. Their scheme, based on the follow-the-perturbed-leader (FTPL)  algorithm, requires very strong approximation guarantee from the approximation oracle, namely, a fully polynomial time approximation scheme (FPTAS), and requires an approximation that improves with time.  \cite{balcan2006approximation} used the same approach in the context of mechanism design. 

\cite{kalai2005efficient} also proposed a specialized reduction that works under certain conditions on the approximation oracle, satisfied by some known algorithms for problems such as MAX-CUT. \cite{fujita2013combinatorial} further gave more general reductions that apply to problems whose approximation algorithms are based on convex relaxations of mathematical programs. Their scheme is also based on the FTPL method. 

Recent advancements on 
black-box online-to-offline reductions %from online to offline %were considered in \cite{kalai2005efficient}, following \cite{hannan1957approximation}, and more recent advancements 
were made in \citep{kakade2009playing,dudik2016oracle,garber2017efficient}. \cite{hazan2016computational} showed that efficient reductions are in general impossible, unless special structure is present. In the settings we consider this special structure is a linear cost function over the space.  

Our algorithms fall into one of two templates. The first is the online mirror descent method, which is an adaptive version of the follow-the-regularized-leader (FTRL) algorithm. The second is the continuous multiplicative weight update method, which dates back to Cover's portfolio selection method \citep{cover1991universal} and Vovk's aggregating algorithm \citep{Vovk:1990:AS:92571.92672}.  The reader is referred to the books \citep{cesa2006prediction,shalev2012online,hazan2016introduction} for details and background on these prediction frameworks.   
We also make use of polynomial-time algorithms for sampling from log-concave distributions \citep{lovasz2007geometry}.

\section{Preliminaries}\label{sec:preliminary}

We use $\|x\|$ to denote the Euclidean norm of a vector $x$.
For $x\in \reals^d$ and $r>0$, denote by $B(x, r)$ the Euclidean ball in $\reals^d$ of radius $r$ centered at $x$, i.e., $B(x, r) := \{x'\in\reals^d: \|x-x'\|\le r\}$.
For $\S, \S' \subseteq \reals^d$, $\beta \in \reals$, $y\in \reals^d$ and $A\in \reals^{d'\times d}$, define $\S+\S' := \{x+x': x\in \S, x'\in \S'\}$, $\beta \S := \{\beta x: x\in \S\}$, $x+\S := \{x+y: y\in \S\}$, and $A\S := \{Ax: x\in\S\}$.
%For $\S \subseteq \reals^d$ and $x\in \reals^d$, define $x+\S := \{x+y: y\in \S\}$.
The convex hull of $\S\subseteq \reals^d$ is denoted by $\ch(\S)$.
Denote by $\vol{\S}$ the volume (Lebesgue measure) of a set $\S \subseteq \reals^d$.
Denote by $\Delta^{k-1}$ the probability simplex in $\reals^k$, i.e., $\Delta^{k-1} := \left\{x\in\reals^k: x_i\ge 0 \ (\forall i), \sum_i x_i = 1 \right\}$.

A set $\C \subseteq \reals^d$ is called a \emph{cone} if for any $\beta\ge0$ we have $\beta \C \subseteq \C$. 
For any $\S \subseteq \reals^d$, define the \emph{dual cone} of $\S$ as $\S^\circ := \left\{y\in\reals^d: x^\top y\geq 0,\ \forall x\in \S \right\}$.
$\S^\circ$ is always a convex cone, even when $\S$ is neither convex nor a cone.

For any closed set $\S \subseteq \reals^d$, define $\Pi_\S:\reals^d\to \S$ to be the \emph{projection} onto $\S$, namely $\Pi_\S(x):= \argmin_{x'\in \S} \norm{x'-x}^2 .$
The well-known Pythagorean theorem characterizes an important property of projections onto convex sets:

\begin{lemma}[Pythagorean theorem] \label{lem:pythagorean}
	For any closed convex set $\S \subseteq \reals^d$, $x\in \reals^d$ and $y\in \S$, we have
	$(\Pi_{\S}(x) - x)^\top (\Pi_{\S}(x) - y) \le 0$, or equivalently, $\norm{x-y}^2 \ge \norm{\Pi_{\S}(x) - x}^2 + \norm{\Pi_{\S}(x) - y}^2$.
\end{lemma}

\begin{definition}
	A function $f: \A \to \reals$ ($\A\subseteq \reals^d$) is \emph{Legendre} if
	\begin{itemize}
		\item $\A$ is convex;
		\item $f$ is strictly convex with continuous gradient defined over $\A$'s interior $\interior(\A)$;
		\item for any sequence $x_1, x_2, \cdots \in \A$ converging to a boundary point of $\A$, $\lim_{n\to\infty}\norm{\nabla f(x_n)} = \infty$.
	\end{itemize}
\end{definition}

\begin{definition}
For a Legendre function $\varphi: \A  \to \reals$, the \emph{Bregman divergence} with respect to $\varphi$ is defined as $D_\varphi(x, y) := \varphi(x) - \varphi(y) - \nabla \varphi(y)^\top(x-y)$ ($\forall x, y\in \A$).
\end{definition}

%\begin{definition}
%The \emph{conjugate} of a function $\varphi:\reals^d \to \reals$ is defined as $\varphi^*(y):=\sup_{x} \left\{y^\top x-\varphi(x)\right\}$. 
%\end{definition}

The Pythagorean theorem can be generalized to projections with respect to a Bregman divergence (see e.g. Lemma 11.3 in \citep{cesa2006prediction}):
\begin{lemma}[Generalized Pythagorean theorem] \label{lem:generalized-pythagorean}
	For any closed convex set $\S \subseteq \reals^d$, $x\in \reals^d$, $y\in \S$, and any Legendre function $\varphi:\reals^d\to \reals$, letting $z = \argmin_{x' \in \S} D_\varphi(x', x)$, we must have
	$D_\varphi(y, x) \ge D_\varphi(y, z) + D_\varphi(z, x)$.
\end{lemma}

\paragraph{Log-concave distributions.}
A distribution over $\reals^d$ with a density function $f$ is said to be \emph{log-concave} if $\log(f)$ is a concave function.
For a convex set $\S$ equipped with a membership oracle, there exist polynomial-time algorithms for sampling from any log-concave distribution over $\S$ \citep{lovasz2007geometry}. This can be used to approximately compute the mean of any log-concave distribution. %, up to accuracy $\eps$ in time $\poly(d,\frac{1}{\eps})$.

We have the following classical result  which says that every half-space close enough to the mean of a log-concave distribution must contain at least constant probability mass.
For simplicity, we only state and prove the result for isotropic (i.e., identity covariance) log-concave distributions, but the result can be easily generalized to allow arbitrary covariance.

%The next Lemma shows that for our purpose, picking $\eps = \frac{1}{\poly(d, T)}$ is sufficient. 

%For every $w, b \in \mathbb{R}^d$, let $H := \{ x \in \mathbb{R}^d, \langle x, w \rangle - b \geq 0 \}$ be a half-space and $\partial H = \{ x \in \mathbb{R}^d, \langle x, w \rangle - b = 0 \}$ be the boundary. We have the following Lemma about the near center cuts of log-concave distributions:
\begin{lemma} \label{lem:center_cut} %[Near center cuts of log-concave distributions]
Consider any isotropic (identity covariance) log-concave distribution $p$ over $\mathbb{R}^d$ with mean $x^*$. Then for any half-space $H$ such that $\norm{ x^* - \Pi_{H}(x^*) } \le \frac{1}{2e}  $, we have $\int_{H} p(x) \dx \ge \frac{1}{2e} $.
\end{lemma}

% For our purpose, it is sufficient to estimate the mean of $p$ up to additive accuracy $O(1/T)$ with out hurting the loss. Thus, taking $\eps = \frac{1}{T}$ suffices for the regret of our algorithm. Moreover, using this Lemma, we know that taking $\eps = \frac{1}{\poly(d)}$ suffices to get a half space to remove constant mass from $p_t$, hence not hurting the number of oracle calls.

%Therefore for simplicity of presentation, in the rest part of this paper, we assume we can compute an approximation of the mean of a log concave distribution, 

The proof of Lemma~\ref{lem:center_cut} is given in Appendix~\ref{app:proof-prelim}.
As an implication, we have the following lemma regarding mean computation of a log-concave distribution, which is useful in this paper.

\begin{lemma}\label{lem:center-cut-to-use}
	For any log-concave distribution $p$ in $\reals^d$ with mean $x^*$, whose support $\textrm{supp}(p)$ is in $B(0,R)$ ($R>0$), and any $\epsilon>0$ and $0<\delta<1$, it is possible to compute a point $\tilde{x}^*$ in $\poly\left(d, \frac1\epsilon, \log \frac1\delta \right)$ time such that with probability at least $1-\delta$ we have: %and ignore the additional regret brought by the approximation. In detail, we assume
	\begin{enumerate}
		\item %For all log-concave distribution $p$ with $\textrm{supp}(p)\subseteq B(0,R)$, 
		$\norm{\tilde{x}^*-x^*} \le R\epsilon$;
		\item for any half space $H$ containing $\tilde{x}^*$, $\int_{H}p(x) \dx \ge \frac{1}{2e}$.
	\end{enumerate}
\end{lemma}

For our purpose in this paper, it always suffices to choose $\epsilon = \frac1T$ and  $\delta = \frac{1}{\poly(T)}$ ($T$ being the total number of rounds) without hurting our regret bounds.
Therefore, for ease of presentation, we will assume that we can compute the mean of bounded-supported log-concave distributions exactly.

\section{Online Improper Linear Optimization with an Improper Optimization Oracle}

Now we describe the problem setting we consider in this paper.
Let $\K, \K^* \subseteq B(0, R)$ ($R>0$) be two compact subsets of $\reals^d$,
and let $W \subseteq \reals^d$ be a convex cone.
Suppose we have an \emph{improper linear optimization oracle} $\O_{\K, \K^*}:W\to \K$, which 
given an input $v \in W$ can output a point $\O_{\K, \K^*}(v) \in \K$ such that 
$$ v^\top \O_{\K, \K^*}(v) \leq \min_{x^* \in \K^* }  v^\top x^*. $$
In other words, it performs linear optimization over $\K^*$ but is allowed to output a point from a (possibly different) set $\K$.
Note that this implicitly requires that $\K$ ``dominates'' $\K^*$ in all directions in $W$, that is, for all $v \in W$ we must have $\min_{x\in \K} v^\top x\leq \min_{x^*\in \K^*} v^\top x^*$.

%Let $W$ be a convex cone in $\reals^d$ and $\K,\K^*\subseteq B(0,R)\subsetneq \reals^d$, such that $\K$ dominates $\K^*$ in all directions in $W$. 

%We use $\mathcal{K} \subseteq \mathbb{R}^d$ to denote the set of feasible points that we play at each iteration, $W \subseteq \mathbb{R}^d$ to denote the set of all possible linear directions the adversary can pick. We define $B(0, L)$ as the sphere in $\mathbb{R}^d$ of radius $L$ centered at zero. Let $f_1, \cdots, f_T \in W \cap B(0, L)$ be the adversarially chosen linear functions, and let $x_1, \cdots, x_T \in \mathcal{K}$ be the corresponding points we pick at these iterations. We  use $\| v \|$ to denote the $\ell_2$-norm of a vector $v$. 

\paragraph{Online improper linear optimization.}
Consider a repeated game with $T$ rounds. In round $t$, the player chooses a point $x_t \in \K$ while an adversary chooses a loss vector $f_t \in W \cap B(0, L)$ ($L>0$), and then the player incurs a loss $f_t^\top x_t$.
The goal for the player is to have a cumulative loss that is comparable to that of the best single decision in hindsight.

We assume that the player only has access to the optimization oracle $\O_{\K, \K^*}$.
Therefore, it is only fair to compare with the best decision in $\K^*$ in hindsight.
The (improper) regret over $T$ rounds is defined as
\begin{equation*} %\label{eqn:improper-regret}
\Reg_{\K, \K^*}(T) := \sum_{t=1}^T f_t^\top x_t - \min_{x^* \in \K^*} \sum_{t=1}^T f_t^\top x^*.  
\end{equation*}
We sometimes treat $f_t$ as a function on $\reals^d$, i.e., $f_t(x) := f_t^\top x$.

\paragraph{Full information and bandit settings.}
We consider both \emph{full information} and \emph{bandit} settings.
In the full information setting, after the player makes her choice $x_t$ in round $t$, the entire loss vector $f_t$ is revealed to the player;
in the bandit setting, only the loss value $f_t(x_t)$ is revealed to the player.

\paragraph{$\alpha$-regret minimization with an approximation oracle.}

The problem of online linear optimization with an approximation oracle considered by \cite{kakade2009playing} and \cite{garber2017efficient} is a special instance in our online improper linear optimization framework.
In this problem, the player has access to an \emph{approximate linear optimization oracle} $\O_\K^\alpha$ over $\K$ ($\alpha>1$), which given a direction $v \in W$ as input can output a point $\O_\K^\alpha(v) \in \K$ such that 
\begin{equation*}
v^\top \O_{\K}^\alpha(v) \leq \alpha \cdot \min_{x \in \K }  v^\top x. 
\end{equation*}
In this setting we will consider $\K \subset \reals_+^d$ and $W = \reals_+^d$; many combinatorial optimization problems with efficient approximation algorithms fall into this regime.
The goal in the online problem
is therefore to minimize the \emph{$\alpha$-regret}, defined as
\begin{equation*} %\label{eqn:alpha-regret}
\Reg^\alpha_\K(T) := \sum_{t=1}^T f_t^\top x_t - \alpha \min_{x \in \K} \sum_{t=1}^T f_t^\top x.  \end{equation*}
To see why this is a special case of online improper linear optimization, note that we can take $\K^* = \alpha \K$ and then the approximation oracle $\O_\K^\alpha$ is equivalent to $\O_{\K, \alpha \K}$ and the $\alpha$-regret $\Reg^\alpha_\K(T)$ is equal to the improper regret $\Reg_{\K, \alpha \K}(T)$.

\section{Efficient Online Improper Linear Optimization via Online Mirror Descent}\label{sec:OMD}

In this section, we give an efficient online improper linear optimization algorithm (in the full information setting) based on
online mirror descent (OMD) equipped with a strongly convex regularizer $\varphi$, which achieves $O(\sqrt{T})$ regret when the regularizer $\varphi$ and the domain of linear loss functions $W$ satisfy the \emph{pairwise non-negative inner product (PNIP)} property (Definition~\ref{def:pnip}). 
This property holds for many interesting domains with appropriately chosen regularizers. Notable examples include the non-negative orthant $\reals^d_+$, the positive semidefinite matrix cone, and the Lorentz cone $L_{d+1}=\{(x,z)\in \reals^{d}\times\reals:\|x\|_2\le z\}$.

\begin{definition}[Pairwise non-negative inner product] \label{def:pnip}
For a twice-differentiable Legendre function $\varphi:\A \to \reals$ with domain $\A\subseteq\reals^d$ and a convex cone $W\subseteq \reals^d$, we say $(\varphi,W)$ satisfies the \emph{pairwise non-negative inner product (PNIP)} property, if for all $w, w' \in W$ and $ H\in \ch(\mathcal{H})$, where $\mathcal{H}=\{\nabla^2 \varphi(x):x\in \A\}$, it holds that $w^\top H^{-1} w'\geq 0$.
\end{definition}
%Some examples are provided in Appendix~\ref{sec:pni-example}.

%\section{Examples for PNI condition} \label{sec:pni-example}

\paragraph{Examples.}
$(\varphi, W)$ satisfies the PNIP property if:
\begin{itemize}
	\item $\varphi(x)=\frac{1}{2}\norm{x}^2$ (with domain $\reals^d$) and  $W\subseteq W^\circ$;
	\item $\varphi(x)=\sum_{i=1}^d x_i(\log x_i-1)$ (with domain $\reals_+^d$) and $W=\reals^d_+$;
	\item $\varphi(x)= \frac{1}{2}x^\top Q^{-1}x$ (with domain $\reals^d$),  where $Q = MM^\top$, $M\in\reals^{d\times d}$ is an invertible matrix, and $W = (M^\top)^{-1}\reals_+^d$.
    %$Q=\sum_{i=1}^d q_iq_i^\top$, $q_1, \ldots, q_d\in\reals^d_+$ are linearly independent, and $W \subseteq \reals^d_+$.
	This is useful in our bandit algorithm in Section~\ref{sec:bandit}.
\end{itemize}

%\begin{example}
%	If $\varphi(x)=\frac{1}{2}\norm{x}^2$,  $W\subseteq W^\circ$,  then $(\r,W)$ satisfies PNI.
%\end{example}
%\begin{example}
%	If $\varphi(x)=\sum_{i=1}^d x_i(\ln x_i-1), W=\reals^d_+$, then  $(\r,W)$ satisfies PNI.
%\end{example}
%\begin{example}
%	If $\varphi(x)= \frac{1}{2}x^\top Q^{-1}x$,  where $Q=\sum_{i=1}^d q_iq_i^\top$, $q_1, \ldots, q_d\in\reals^d_+$ are linearly independent, and $W = \reals^d_+$, then $(\r,W)$ satisfies PNI. 
%\end{example}

\subsection{Online Mirror Descent with a Projection-and-Decomposition Oracle}

We first show that assuming the availability of a \emph{projection-and-decomposition (PAD)} oracle (Definition~\ref{def:projection_oracle}), we can implement a variant of the OMD algorithm that achieves optimal regret.
In Section~\ref{subsec:pad}, we show how to construct a PAD oracle  using the oracle $\O_{\K,\K^*}$.
In Section~\ref{subsec:omd-number-calls}, we bound the number of oracle calls to $\O_{\K,\K^*}$ in our algorithm.

\begin{definition}[Projection-and-decomposition oracle]\label{def:projection_oracle}
A \emph{projection-and-decomposition (PAD)} oracle onto $\K^*$,
$\pad(y, \epsilon, W, \varphi)$,
 is defined as a procedure that given $y \in \reals^d$, $\epsilon>0$, a convex cone $W$ and a Legendre function $\varphi$ produces a tuple $(y',V,p) $, where $y' \in \reals^d$, $V = (v_1, \ldots, v_k) \in \reals^{d\times k}$ and $p=(p_1, \ldots, p_k)^\top\in\Delta^{k-1}$, such that:
\begin{enumerate}
\item $y'$ is ``closer'' to $\K^*$ than $y$ with respect to the Bregman divergence of $\varphi$ (and hence is an ``infeasible projection''):
$ \forall x^* \in \K^* , \  D_\varphi(x^*,y') \le D_\varphi(x^*,y)   $;
\item $v_1, \ldots, v_k\in \K$, and $\sum_{i=1}^k p_iv_i$ is a point that ``almost dominates'' $y'$ in all directions in $W$. In other words, there exists $c\in W^\circ$ such that
$  \norm{ \sum_{i=1}^k p_i v_i + c - y'  } \leq \epsilon $.
\end{enumerate}
\end{definition}

The purpose of the PAD oracle is the following.
Suppose the OMD algorithm tells us to play a point $y$.
Since $y$ might not be in the feasible set $\K$, we can call the PAD oracle to find another point $y'$ as well as a distribution $p$ over points $v_1, \ldots, v_k \in \K$.
The first property in Definition~\ref{def:projection_oracle} is sufficient to ensure that playing $y'$ also gives low regret, and the second property further ensures that we have a distribution of points in $\K$ that suffers less loss than $y'$ for every possible loss function so we can play according to that distribution.

%The intuition behind the above property 2 is that when we want to play some $x\notin \K$, we can instead find a distribution whose support is in $\K$ and the mean suffers less loss than $x$ does for every possible loss function. 

Using the PAD oracle, we can apply OMD as in Algorithm~\ref{alg:OMD}. Theorem~\ref{thm:OMD_regret} gives its regret bound.

\begin{algorithm}[htbp]
\caption{Online Mirror Descent using a Projection-and-Separation Oracle}
\label{alg:OMD}
\begin{algorithmic}[1]
	\REQUIRE Learning rate $\eta>0$, tolerance $\epsilon>0$, regularizer $\varphi$, convex cone $W$, time horizon $T \in \mathbb N_+$
	\STATE $y_1\gets \argmin_{y\in \mathrm{Dom}(\varphi)}\varphi(y)$.
	\FOR{$t=1$ \TO $T$}
    	\STATE  $(x_{t}, V,p)\gets \pad(y_{t},\epsilon,W,\varphi) $
		\STATE Play $\tilde x_t = v_i$ with probability $p_i$ ($i\in[k]$), where $V = (v_1, \ldots, v_k)$, and observe the loss vector $f_t$
		\STATE $\nabla \varphi(y_{t+1}) \leftarrow \nabla \varphi (x_t) - \eta f_t $ 
	\ENDFOR
\end{algorithmic}
\end{algorithm}

\begin{theorem}\label{thm:OMD_regret}
	Suppose $(\varphi,W)$ satisfies the PNIP property (Definition~\ref{def:pnip}).
	%Denote by $\tilde x_t$ the point played by Algorithm \ref{alg:improper_CMWU} in round $t$.
	Then for any $\epsilon, \eta>0$, Algorithm \ref{alg:OMD} satisfies the following regret guarantee: 
	
%For any learning rate $\eta>0$, tolerance $\epsilon>0$ and any loss sequence $\{f_t\}_{t=1}^T\in (\reals^d_+\cap B(0,L))^T$, any domain $\K\cup \K^* \subset B(0,R)$, assuming $(\varphi,W)$ satisfies PNI property,  Algorithm $\ref{alg:OMD}$ has the following regret guarantee: 
\begin{align*}
\forall x^* \in \K^*: \quad  \E\left[\sum_{t=1}^T ( f_t(\tilde x_t)-f_t(x^*) ) \right] \le \frac{1}{\eta}\left(\varphi(x^*)-\varphi(y_1)+\sum_{t=1}^T D_{\varphi}(x_t,y_{t+1})\right) +\epsilon LT.
\end{align*}
%using at most $5d\ln \frac{R}{\epsilon}$ calls of oracle $\O_{\K,\K^*}$ per round.

In particular, if $\varphi$ is $\mu$-strongly convex and $A \ge \max_{x^* \in \K^*} (\varphi(x^*)-\varphi(y_1))$,
setting $\epsilon = \frac{R}{T}$ and $\eta = \frac1L\sqrt{\frac{2\mu A}{T}}$, we have%\footnote{In Theorem~\ref{thm:omd-number-calls}, we show that each round of Algorithm~\ref{alg:OMD} only needs to call $\O_{\K, \K^*}$ for at most $5d \log \left( 4T + 4\sqrt{3AT^3/(\mu R^2)} \right)$ times.}
\[
\forall x^* \in \K^*: \quad  \E\left[\sum_{t=1}^T ( f_t(\tilde x_t)-f_t(x^*) ) \right] \le
 L\sqrt{\frac{2AT }{\mu}}+LR.
\]
%using at most $5d\ln (T)$ calls of oracle $\O_{\K,\K^*}$ per round.
\end{theorem}

\begin{proof}
	First, for any fixed round $t\in[T]$, let $(x_t, V, p)$ be the output of $\pad(y_t, \epsilon, W, \varphi)$ in this round.
	We know by the second property of the PAD oracle that there exists $c \in W^\circ$ such that $\norm{\sum_i p_iv_i +c - x_t} \le \epsilon$.
	Since $\tilde x_t$ is equal to $v_i$ with probability $p_i$, letting $\overline x_t := \E[\tilde{x}_t]=\sum_ip_iv_i$, we have
	\begin{equation} \label{eqn:omd-inproof-1}
	f_t(\overline x_t) -f_t(x_t) = \E\left[ f_t(\tilde x_t) - f_t(x_t) \right]  =  f_t\left(\sum_i p_iv_i - x_t\right) \le 
	f_t\left(\sum_i p_iv_i - x_t + c\right) \le \epsilon L.
	\end{equation}

	 We make use of the following properties of Bregman divergence, which can be verified easily (see e.g. Section 11.2 in \citep{cesa2006prediction}):
	 \begin{equation} \label{eqn:three-point-bregman}
	 \forall x, y, z:\quad (x-y)^\top (\nabla \varphi(z) - \nabla \varphi(y) ) = D_\varphi( x,y) - D_\varphi( x,z) + D_\varphi(y,z).
	 \end{equation}
	
%	 \begin{enumerate}[(i)]
%	 	\item $(x-y)^\top (\nabla \varphi(z) - \nabla \varphi(y) ) = D_\varphi( x,y) - D_\varphi( x,z) + D_\varphi(y,z)$;
%%	 	\item $D_{\varphi}(x, y) = D_{\varphi^*}(\nabla\varphi(y), \nabla\varphi(x))$.
%	 	\item for $\mu$-strongly convex $\varphi$, $D_{\varphi}(x, y) \leq \frac{1}{2 \mu} \| \nabla \varphi(x) - \nabla \varphi(y) \|^2 $. 
%	 \end{enumerate}

	Consider any $x^* \in \K^*$. We have
	\begin{equation}\label{eqn:omd-inproof-1.5}
	\begin{aligned}
		&\sum_{t=1}^T (f_t(x_t) - f_t(x^*))  \\
		 =\ & \sum_{t=1}^T  \frac{1}{\eta} \left( \nabla \varphi( x_{t} )  - \nabla \varphi( y_{t+1}) \right)^\top (x_t - x^*) & \mbox{(by algorithm definition)}  \\
		 =\ & \frac{1}{\eta} \sum_{t=1}^T  \left( D_\varphi ( x^* , x_t) - D_\varphi(x^* , y_{t+1} ) + D_\varphi(x_t , y_{t+1} ) \right) & \mbox{(by \eqref{eqn:three-point-bregman})}   \\
		 \le\ &  \frac{1}{\eta}\sum_{t=1}^T  \left( D_\varphi ( x^* , y_t) - D_\varphi(x^* , y_{t+1} ) + D_\varphi(x_t , y_{t+1} ) \right)  & \mbox{(by property of the PAD oracle)}  \\
		=\ &   \frac{1}{ \eta} \left( D_\varphi ( x^* , y_1) - D_\varphi ( x^* , y_{T+1}) +  \sum_{t=1}^TD_\varphi(x_t, y_{t+1}) ) \right).  & \mbox{(by telescoping)} %\\
		 %\le\ & \frac{1}{ \eta} \left( D_\varphi ( x^* , y_1)  + \frac{1}{2 \mu} \sum_{t=1}^T \| \eta f_t\|^2  \right)  & \mbox{(by algorithm definition)} \\
		% \le\ & \frac{1}{ \eta} \left( \varphi(x^*)-\varphi(y_1) + \sum_{t=1}^T D_{\varphi^*}(\nabla \varphi(x_t)-\eta f_t,\nabla \varphi(x_t)) \right). & \mbox{(by optimality condition $\nabla\varphi(y_1)^\top(x^*-y_1)\ge0$)}
		\end{aligned}
	\end{equation}
	Combining \eqref{eqn:omd-inproof-1} and \eqref{eqn:omd-inproof-1.5}, we can bound the expected improper regret of Algorithm~\ref{alg:OMD} as
	\begin{equation} \label{eqn:omd-inproof-2}
	\begin{aligned}
		%%\E\left[\sum_{t=1}^T ( f_t(\tilde x_t)-f_t(x^*) ) \right] %%
		\forall x^*\in\K^*:\quad
        &\E\left[\sum_{t=1}^T ( f_t(\tilde x_t)-f_t(x^*) ) \right] = \sum_{t=1}^T ( f_t(\overline x_t)-f_t(x^*) )\\
        \le\ & \frac{1}{ \eta} \left( D_\varphi ( x^* , y_1) - D_\varphi ( x^* , y_{T+1}) +  \sum_{t=1}^TD_\varphi(x_t, y_{t+1}) ) \right) + \epsilon LT.
        %\frac{1}{ \eta} \left( \varphi(x^*)-\varphi(y_1) + \frac{\eta^2 }{2 \mu }\sum_{t=1}^T \|f_t\|^2 \right)  + \epsilon LT.
        \end{aligned}
	\end{equation}

	By the optimality condition $\nabla\varphi(y_1)^\top(x^*-y_1)\ge0$, we have 
	\begin{align}\label{eqn:omd-inproof-1.75}
	D_\varphi ( x^* , y_1) \le \varphi(x^*)-\varphi(y_1).
	\end{align}
	Plugging \eqref{eqn:omd-inproof-1.75} into \eqref{eqn:omd-inproof-2} and noting $D_\varphi(x^*, y_{T+1})\ge0$, we finish the proof of the first regret bound.
	
%	If $\varphi$ is $\mu$-strongly convex, then $\varphi^*$ is $\frac1\mu$-smooth \citep{kakade09duality}, which implies
%	\begin{align}\label{eqn:omd-inproof-3}
%	D_{\varphi^*}(\nabla \varphi(x_t)-\eta f_t,\nabla \varphi(x_t))
%	\le \frac{\|\eta f_t\|^2}{2\mu}\leq \frac{\eta^2 L^2}{2\mu}.
%	\end{align} 

	When $\varphi$ is $\mu$-strongly convex, we have the following well-known property:\footnote{See \url{http://xingyuzhou.org/blog/notes/strong-convexity} for a proof.}
	\begin{align*}
	D_\varphi(x, y) \le \frac{1}{2\mu} \norm{\nabla \varphi(x) - \nabla \varphi(y)}^2.
	\end{align*}
	Then by the definition in Algorithm~\ref{alg:OMD} we have
	\begin{equation} \label{eqn:omd-inproof-3}
	\forall t\in [T]:\quad D_\varphi(x_t, y_{t+1}) \le \frac{1}{2\mu} \norm{\nabla \varphi(x_t) - \nabla \varphi(y_{t+1})}^2 = \frac{1}{2\mu} \norm{\eta f_t}^2 \le \frac{\eta^2L^2}{2\mu}.
	\end{equation}
	From the above inequality and the choices of parameters $\epsilon = \frac RT$ and $\eta = \frac1L \sqrt{\frac{2\mu A}{T}}$, we have
	 \begin{equation*} 
	 \E\left[\sum_{t=1}^T ( f_t(\tilde x_t)-f_t(x^*) ) \right] \le \frac{A}{\eta} + \frac{\eta L^2 T}{2\mu} +  LR
	 \le L \sqrt{\frac{2AT}{\mu}} + LR . \qedhere
	 \end{equation*}
\end{proof}

For the problem of $\alpha$-regret minimization using an $\alpha$-approximation oracle, we have the following regret guarantee, which is an immediate corollary of Theorem~\ref{thm:OMD_regret}.
\begin{corollary}
If $W\subseteq \reals^d_+$, $\K\subseteq B(0,R)$, $\K^*=\alpha\K$, $\varphi(x)= \frac{1}{2}\norm{x}^2$, setting $\epsilon = \frac{\alpha R}{T}$, $\eta=\frac{\alpha R}{L\sqrt{T}}$, Algorithm~\ref{alg:OMD} has the following regret guarantee:
\[
\forall x^* \in \K^*: \quad  \E\left[\sum_{t=1}^T  f_t(\tilde x_t) - \alpha \sum_{t=1}^T f_t(x^*) \right] \le \alpha LR(\sqrt{T}+1).
\]
%using at most $5d\ln(T)$ calls of oracle $\O_{\K,\alpha \K}$ per round.
\end{corollary}

\subsection{Construction of the Projection-and-Decomposition Oracle} \label{subsec:pad}

Now we show how to construct the PAD oracle using the improper linear optimization oracle $\O_{\K, \K^*}$. Our construction is given in Algorithm~\ref{alg:infeasible_projection}.

%In this section we prove Theorem \ref{thm:projection} concerning approximate projections, computed by Algorithm~\ref{alg:infeasible_projection}. We will use $\D(x,\bv,W)$ to represent the solution to the following convex program given $\bv=\{v_i\}_{i=1}^k\in (\reals^d)^k$:
%        \begin{equation} \label{eqn:convex-program}
%        	\D(x,\bv,W):= \argmin\limits_{\bp\in \Delta^{k-1},c\in W^\circ} \|\sum_{i=1}^k p_iv_i+c-x\|
%        \end{equation}

\begin{algorithm}[htbp]
\caption{Projection-and-Decomposition Oracle, $\pad(y,\epsilon,W,\varphi) $    }\label{alg:infeasible_projection}
\begin{algorithmic}[1]
	\REQUIRE Point $y\in \reals^d$, tolerance $\epsilon>0$, convex cone $W$, regularizer $\varphi$, %number of iterations $k \in \mathbb N_+$
	\ENSURE $(y',V,p) $, where $y' \in \reals^d$, $V = (v_1, \ldots, v_k) \in \reals^{d\times k}$ for some $k$ such that $v_i \in \K$ ($\forall i \in [k]$), and $p=(p_1, \ldots, p_k)^\top \in\Delta^{k-1}$
	%$y' = y_{k+1}$,$(v_1,\ldots,v_k)$, $(p_1,\ldots,p_k)$, where $\bp$ is a distribution on $\bv$ and $k=5d\ln \frac{2R}{\epsilon}$.
    \STATE $W_1 \gets W\cap B(0,1)$, $z_1 \gets y$
    \STATE $i\gets0$
    \WHILE{$i < 5d \log \frac{2(R+\norm{z_{i+1}})}{\epsilon}$}
	%\FOR{$i=1$ \TO $k$}
    	\STATE $i\gets i+1$
		\STATE $w_i \gets \frac{\int_{ W_i}w \dw}{\vol{W_i}}$. 
		\STATE $v_i \gets \O_{\K,\K^*}(w_i)$.
		%\STATE $x_{t+1} = x_{t} + \max\{0,w_t^T(p_t-x_t)\}\frac{w_t}{\|w_t\|^2}$. \mbox{Now we have $w_{t}^T(p_t-x_{t+1})\leq 0$}
        \STATE $z_{i+1} \gets \argmin\limits_{z \in \reals^d, w_i^\top(z-v_i)\ge0} D_\varphi(z,z_i)$ 
		\STATE $W_{i+1} \gets W_i \cap \{w\in \reals^d: w^\top (v_i-z_{i+1}) \ge 0\}$
	\ENDWHILE
    \STATE $k \gets i$
    \STATE Solve $\min\limits_{p\in \Delta^{k-1}, c\in W^\circ} \norm{\sum_{i=1}^k p_iv_i+c-z_{k+1}}$ to get $p$
    \RETURN $y' = z_{k+1}, V = (v_1, \ldots, v_k), p$
\end{algorithmic}
\end{algorithm}

\begin{theorem} \label{thm:projection}
	Suppose $(\varphi, W)$ satisfies the PNIP condition (Definition~\ref{def:pnip}) and $\varphi$ is $\mu$-strongly convex.
	Then for any $y\in \reals^d$ and $\epsilon \in (0, R]$, Algorithm~\ref{alg:infeasible_projection} must terminate in $k \le \left\lceil 5d \log  \frac{4R+2\sqrt{\frac{2}{\mu}\min_{x^* \in \K^*} D_\varphi(x^*, y)}}{\epsilon} \right\rceil$ iterations, and it correctly implements the projection-and-decomposition oracle $\pad(y, \epsilon, W, \varphi)$, i.e., its output $(y',V,p)$ satisfies the two properties in Definition \ref{def:projection_oracle}.
%Algorithm \ref{alg:infeasible_projection} computes the projection $(y',\bv,\bp) = \widetilde{\Pi}_{\K,\K^*}(y,\epsilon,W,\varphi) $ using $5d \log \frac{2R}{\eps}$ oracle calls. The tuple $(y',\bv,\bp)$ satisfies the two properties in Definition \ref{def:projection_oracle}.
% \begin{enumerate}
% \item
% $ \forall y^* \in \alpha\K \ , \  D_\varphi(y^*,y') \leq D_\varphi(y^*,y)   $
% \item $v_i\in \K,\ \forall i=[k]$. And $\{p_i\}_{i=1}^k$ represents a convex combination of  $\bv$ which `almost' dominates $x'$. In other words, there exists $c\in W$, s.t.
% $$  \| \sum_{i=1}^k p_i v_i + c - y'  \| \leq \eps $$
% \end{enumerate}
\end{theorem}

We break the proof of Theorem~\ref{thm:projection} into several lemmas.

\begin{lemma} \label{lem:pad-move-direction}
If $(\varphi, W)$ satisfies the PNIP condition (Definition~\ref{def:pnip}), 
%namely $\forall w',w\in W$, and $\forall H\in CH(\mathcal{H})$, where $\mathcal{H}=\{\nabla~^2 \varphi(x): x\in \textrm{Dom}(\varphi)\}$, it holds that $w^\top H w'\geq 0$, 
then $z_1, \ldots, z_{k+1}$ computed in Algorithm~\ref{alg:infeasible_projection} satisfy
 $z_{i+1}-z_i\in W^\circ$ for all $i\in[k]$.
\end{lemma}
\begin{proof}
Since we have $z_{i+1} = \argmin\limits_{z\in\reals^d:w_i^\top (z-v_i)\ge0} D_\varphi(z,z_i)$, by the KKT condition, we have $$0 = \frac{\partial}{\partial z} \left( D_\varphi(z,z_i) - \lambda w_i^\top (z-v_i) \right) \Big|_{z=z_{i+1}} = \nabla \varphi(z_{i+1})-\nabla \varphi(z_i) - \lambda w_i $$ for some $\lambda\ge0$. On the other hand, note that $\nabla \varphi(z_{i+1})-\nabla \varphi(z_i)=\int_{0}^1\nabla^2\varphi(\gamma z_{i+1}+(1-\gamma)z_i)\cdot(z_{i+1}-z_i)\mathrm{d}\gamma=H(z_{i+1}-z_i)$, for some $H\in\ch(\mathcal{H})$, where $\mathcal{H} = \left\{ \nabla^2\varphi(x): x \in \mathrm{Dom}(\varphi) \right\}$. Therefore, for all $ w\in W$ we have $w^\top (z_{i+1}-z_i)=w^\top H^{-1}H(z_{i+1}-z_i)= \lambda w^\top H^{-1}w_i\ge 0$.
This means $z_{i+1}-z_i\in W^\circ$.
\end{proof}
% \begin{lemma}
% If $\varphi(\x)=\sum_{i=1}^d r_i(x^i)$, where $r_i$ are convex functions. Then $\forall 1\leq t\leq k$, $y_{t+1}-y_t\in \reals^d_+$.
% \end{lemma}

% \begin{proof}
% By definition, $y_{t+1} = \argmin\limits_{y:w_t^T(y-v_t)\geq0} D_\varphi(y,y_t)$. By optimality condition, $\nabla D_\varphi(y,y_t)|_{y=y_{t+1}} = \nabla \varphi(y_{t+1})-\nabla \varphi(y_t) = C w$, for some $C>0$. Since $\varphi(x)$ is separable, we have $0\leq Cw_i= r_i'(y_{t+1}^i)- r_i'(y_t^i)$. Note that $r_i'$ nondecreasing, thus $y_{t+1}^i-y_t^1\geq 0$, $y_{t+1}-y_t\in \reals^d_+$.
% \end{proof}

\begin{lemma}\label{lem:pad-running-time}
	Under the setting of Theorem~\ref{thm:projection}, Algorithm~\ref{alg:infeasible_projection} terminates in at most $$\left\lceil 5d \log  \frac{4R+2\sqrt{\frac{2}{\mu}\min_{x^* \in \K^*} D_\varphi(x^*, y)}}{\epsilon} \right\rceil$$ iterations.
\end{lemma}
\begin{proof}
	According to the algorithm,
	for each $i$, $z_{i+1}$ is the Bregman projection of $z_i$ onto a half-space containing $\K^*$, since the oracle $\O_{\K, \K^*}$ ensures $w_i^\top v_i \le w_i^\top x^*$ for all $x^* \in \K^*$.
	Then by the generalized Pythagorean theorem (Lemma~\ref{lem:generalized-pythagorean}) we know $D_\varphi(x^*, z_{i+1}) \le D_\varphi(x^*, z_i)$ for all $x^* \in \K^*$ and $i$. 
	Therefore we have $ D_\varphi(x^*, z_{i}) \le D_\varphi(x^*, z_{1}) = D_\varphi(x^*, y) $
%    $ D_\varphi(x^*, z_{k+1}) \le D_\varphi(x^*, z_{k}) \le \cdots \le D_\varphi(x^*, z_{1}) = D_\varphi(x^*, y) $ 
for all $x^*\in \K^*$ and $i$.
	
	Let $P := \min_{x^* \in \K^*} D_\varphi(x^*, y)$.
	Then there exists $x^* \in \K^*$ such that $P = D_\varphi(x^*, y) \ge D_\varphi(x^*, z_i) \ge \frac{\mu}{2} \norm{x^* - z_i}^2$ for all $i$, where the last inequality is due to the $\mu$-strong convexity of $\varphi$.
	 This implies $\norm{z_i} \le \norm{x^*} + \sqrt{\frac{2P}{\mu}} \le R + \sqrt{\frac{2P}{\mu}}$ for all $i$.
     Therefore, when $i\ge 5d \log\frac{4R+2\sqrt{2P/\mu}}{\epsilon}$, we must have $i\ge 5d \log\frac{2(R+\norm{z_{i+1}})}{\epsilon}$, which means the loop must have terminated at this time. This proves the lemma.
\end{proof}

%A standard analysis for cutting plane method also gives the following lemma:

\begin{lemma}\label{lem:cutting_plane_OMD}
	Under the setting of Theorem~\ref{thm:projection}, 
for all $ w\in W, \|w\|=1$, there exists $i\in[k]$ such that  $w^\top (v_i-y')\le \epsilon$.
\end{lemma}
\begin{proof}%{\lemmaref{lem:cutting_plane_OMD}}
	We assume for contradiction that
	 there exists a unit vector $h\in W$ such that $\min_{i\in [k]} h^\top(v_i-y')> \epsilon$. Note that $\|v_i-y'\|\le \|v_i\|+\|y'\|\le R+\norm{y'}$. Letting $r := \frac{\epsilon}{2(R + \norm{y'})}$, we have 
	\[\forall w\in \frac{h}{2}+ (W\cap B(0,r)):\quad \min_{i\in [k]} w^\top(v_i-y')> 0.\]
	Since $r\le \frac{1}{2}$ for $\epsilon\le R$, we have $\frac{h}{2}+ (W\cap B(0,r)) \subseteq \frac{h}{2}+ (W\cap B(0,1/2)) \subseteq W\cap B(0,1)=W_1$. 
	
	By the algorithm, we know that
	for all $ w\in W_1\setminus W_{k+1}$, there exists $i\in[k]$ such that $w^\top(v_i-z_{i+1})\le 0$. Notice that from Lemma~\ref{lem:pad-move-direction} we know $z_{j+1}-z_{j}\in W^\circ$ for all $j\in[k]$. Thus  for all $ w\in W_1\setminus W_{k+1}$ there exists $i\in[k]$ such that
	$w^\top (v_i-y') = w^\top (v_i-z_{k+1}) \le w^T(v_i-z_{i+1})\le 0$.
	In other words, we have
	\[\forall w\in W_1\setminus W_{k+1}:\quad \min_{i\in [k]} w^\top(v_i-y')\le 0.\]
	
	Therefore, we must have $\frac{h}{2}+ (W\cap B(0,r))\subseteq W_{k+1}$. 
		We also have $\vol{W_{i+1}} \le (1-1/(2e)) \vol{W_i}$ for each $i\in[k]$ from Lemma~\ref{lem:center-cut-to-use}, since $W_{i+1}$ is the intersection of $W_i$ with a half-space that does not contain $W_i$'s centroid $w_i$ in the interior.
		Then we have
		\begin{align*}
		\vol{W_1} &=\vol{W\cap B(0,1)}= r^{-d}\vol{W\cap B(0,r)}\le r^{-d}\vol{W_{k+1}}\\
		&\le r^{-d}(1-1/(2e))^{k}\vol{W_1}< \vol{W_1},
		\end{align*}
		where the last step is due to $k \ge 5d \log\frac1r = 5d \log\frac{2(R+\norm{y'})}{\epsilon} = 5d \log\frac{2(R+\norm{z_{k+1}})}{\epsilon}$, which is true according to the termination condition of the loop. Therefore we have a contradiction.
\end{proof}

We need the following basic property of projection onto a convex cone. The proof is given in Appendix~\ref{app:proof-omd}.
\begin{lemma}\label{lem:diff_in_dual}
	For any closed convex cone $W\subseteq \reals^d$ and any $x\in\reals^d$, we have
	$ \Pi_W(x)-x\in W^\circ$.
\end{lemma}

%The final step of the proof of \cref{thm:projection} is 
The following lemma is a more general version of Lemma 6 in \citep{garber2017efficient}.
\begin{lemma}\label{lem:equivalence} 
	Given %$V = (v_1, \ldots, v_k) \in \reals^{d\times k}$ where 
	$v_1, \ldots, v_k \in \reals^d$, $\epsilon\ge0$ and a convex cone $W \in \reals^d$, for any $x\in \reals^d$, the following two statements are equivalent:
\begin{enumerate}[(A)]
    \item There exists $p = (p_1, \ldots, p_k)^\top \in \Delta^{k-1}$ and $c\in W^\circ$ such that $\norm{\sum_{i=1}^k p_iv_i+c-x}\le \epsilon$.
    \item For all $ w\in W$, $\|w\|=1$, there exists $i\in[k]$ such that $w^\top(v_i-x)\le \epsilon$.
\end{enumerate}
\end{lemma}

\paragraph{Geometric interpretation of Lemma~\ref{lem:equivalence}.}
Before proving Lemma~\ref{lem:equivalence}, we discuss its geometric intuition.
For simplicity of illustration, we only consider $\epsilon=0$ here.
First we look at the case where $W=\reals^d,\ W^\circ = \{0\}$. In this case the lemma simply degenerated to the fact \[x\in \ch(\{v_i\}_{i=1}^k)\Longleftrightarrow \textrm{There is no hyperplane that separates }x \textrm{ and all } v_i\text{'s}.\]
In the general case where $W\subseteq \reals^d$ is an arbitrary convex cone, lemma~\ref{lem:equivalence} becomes 
\[ x\in \ch(\{v_i\}_{i=1}^k)+ W^\circ \Longleftrightarrow \textrm{There is no direction } w\in W \textrm{ such that } w^\top x < w^\top v_i \text{ for all } i.\]
Denote $F:= \ch(\{v_i\}_{i=1}^k)+ W^\circ$.
For the ``$\Rightarrow$'' side,
if $x \in F$, it is clear that for all $w\in W$ we must have $w^\top x \ge w^\top v_i$ for some $i$.
For the ``$\Leftarrow$'' side, if $x\notin F$, then $w = \Pi_F(x)-x$ satisfies $w^\top x < w^\top v_i$ for all $i$. 
%separates $x$ and $\{v_i\}_{i=1}^k$, where $\Pi(x)$ is defined as the projection of $x$ onto convex set $\textrm{CH}(\{v_i\}_{i=1}^k)+W^\circ$. Namely $(\Pi(x)-x)^\top (v_i-x)\leq 0,\quad \forall i\in [k]$.
Moreover it is easy to see $\Pi_F(x)-x\in W$, which completes the proof. %Thus there's a hyperplane whose normal vector is $\Pi_(x)-x\in W$ which separates $x$ and $\{v_i\}_{i=1}^k$.
See Figure~\ref{fig:demon1} for a graphic illustration.

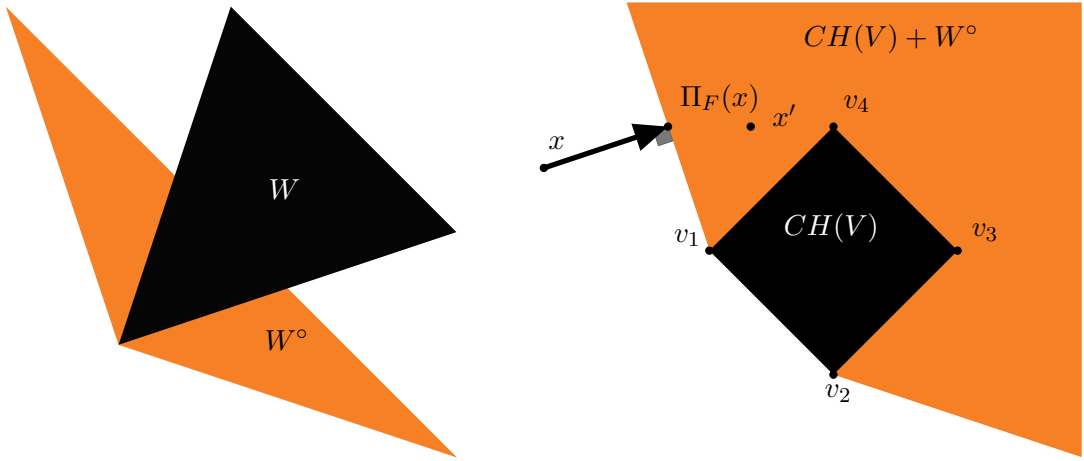
\begin{figure}[htbp] \centering
    \begin{subfigure}[b]{0.49\linewidth}
        \definecolor{ffffff}{rgb}{1,1,1}\definecolor{qvqvqv}{rgb}{0.0196078431372549,0.0196078431372549,0.0196078431372549}\definecolor{fvyqsv}{rgb}{0.9607843137254902,0.5019607843137255,0.1450980392156863}

\begin{tikzpicture}[line cap=round,line join=round,>=triangle 45,x=1cm,y=1cm,scale=0.75]
\clip(-3.5562340966921235,-2.142900763358787) rectangle (7.858320610687037,6.24);
\fill[line width=2pt,color=fvyqsv,fill=fvyqsv,fill opacity=1] (0,0) -- (-2,6) -- (6,-2) -- cycle;\fill[line width=2pt,color=qvqvqv,fill=qvqvqv,fill opacity=1] (0,0) -- (1.9892111959287533,6) -- (5.989211195928753,2) -- cycle;
\draw [color=ffffff](2.485496183206109,3.1112468193384255) node[anchor=north west] {$W$};
\draw (2.409974554707381,0.4464122137404555) node[anchor=north west] {$W^\circ$};
\end{tikzpicture}
        \label{fig:dual_cone}
        \subcaption{Convex cone $W$ and its dual cone $W^\circ$}
    \end{subfigure} %
    \begin{subfigure}[b]{0.49\linewidth}    
        \definecolor{ffffff}{rgb}{1,1,1}\definecolor{fvyqsv}{rgb}{0.9607843137254902,0.5019607843137255,0.1450980392156863}\definecolor{wrwrwr}{rgb}{0.3803921568627451,0.3803921568627451,0.3803921568627451}\begin{tikzpicture}[line cap=round,line join=round,>=triangle 45,x=1cm,y=1cm,scale = 0.55]
\clip(-7.25,-5.2) rectangle (8,6);\fill[line width=2pt,color=fvyqsv,fill=fvyqsv,fill opacity=1] (-2.990979205697039,0) -- (-5,6.003454879086257) -- (6,6.003454879086257) -- (6,-4.996545120913743) -- (0,-3) -- cycle;
%\fill[line width=2pt,color=wrwrwr,fill=wrwrwr,fill opacity=1] (-3,0) -- (0,-3) -- (3,0) -- (0,3) -- cycle;
\fill[line width=2pt,color=black,fill=black,fill opacity=1] (-3,0) -- (0,-3) -- (3,0) -- (0,3) -- cycle;
\draw[line width=0.4pt,color=wrwrwr,fill=wrwrwr,fill opacity=0.85] (-4.344949960274375,2.8850166799085413) -- (-4.2299666401829175,2.5400667196341655) -- (-3.8850166799085413,2.6550500397256243) -- (-4,3) -- cycle; \draw [->,line width=2pt] (-7,2) -- (-4,3);
\draw [color=ffffff](-1.4450338940521152,1.1496116059545012) node[anchor=north west] {$CH(V)$};
\draw (-0.9650942567129601,5.726178862010015) node[anchor=north west] {\bf $CH(V)+W^\circ$};
\draw (-7.1357467367878105,3) node[anchor=north west] {$x$};
\draw (-3.9475762887491377,4.269219248659009) node[anchor=north west] {$\Pi_F(x)$};
\draw (-1.7192851153887752,3.789279611319854) node[anchor=north west] {$x'$};
\draw (-4.1,0.7553754752830524) node[anchor=north west] {$v_1$};
\draw (-0.4508732167067226,-3) node[anchor=north west] {$v_2$};
\draw (3.0972519593363166,0.9) node[anchor=north west] {$v_3$};
\draw (-0.0052149820346500345,4) node[anchor=north west] {$v_4$};
\begin{scriptsize}\draw [fill=black] (-3,0) circle (2.5pt);\draw [fill=black] (0,-3) circle (2.5pt);\draw [fill=black] (3,0) circle (2.5pt);\draw [fill=black] (0,3) circle (2.5pt);\draw [fill=black] (-4,3) circle (2.5pt);\draw [fill=black] (-7,2) circle (2.5pt);\draw [fill=black] (-2,3) circle (2.5pt);
\end{scriptsize}
\end{tikzpicture}
        \label{fig:equivalence}   
        \subcaption{An example for $\ch(V)+W^\circ$, where $V=\{v_i\}_{i=1}^4$.}
    \end{subfigure} 
    \caption{Geometric Interpretation of Lemma~\ref{lem:equivalence}.}
    \label{fig:demon1}
\end{figure}

\begin{proof}[Proof of Lemma~\ref{lem:equivalence}]
		 Suppose (A) holds.
		 Then for any
		  $ w\in W$, $\norm{w}=1$, we have
		  \begin{align*}
		  \min_{i\in[k]} w^\top (v_i-x)&\le w^\top\left(\sum_{i=1}^k p_iv_i-x\right)
		\le w^\top \left(\sum_{i=1}^k p_iv_i+c-x\right)
		\le \norm{w} \cdot \left\|\sum_{i=1}^k p_iv_i+c-x\right\|\le \epsilon.
		\end{align*}
		So we have (A)$\implies$(B).
		
		 Now suppose (B) holds. Let $(p^*,c^*) = \argmin\limits_{p\in\Delta^{k-1}, c\in W^\circ}\norm{\sum_{i=1}^k p_iv_i+c-x}$.
		Since $0\in W$, by the Pythagorean theorem (Lemma~\ref{lem:pythagorean}), we have 
		\[ \left\|\sum_{i=1}^k p^*_iv_i-x+c^*\right\| \ge \left\|\Pi_W\left(\sum_{i=1}^k p^*_iv_i-x+c^*\right)\right\|,\]
		where the equality holds only when $\sum_{i=1}^k p^*_iv_i-x+c^* \in W$.
		Now we claim $\sum_{i=1}^k p^*_iv_i-x+c^*\in W$. Otherwise, letting $c' = c^* + \Pi_W\left(\sum_{i=1}^k p^*_iv_i-x+c^*\right)-\left(\sum_{i=1}^k p^*_iv_i-x+c^*\right)$, by Lemma~\ref{lem:diff_in_dual} we have $c'\in W^\circ$, and furthermore
		\[ \left\|\sum_{i=1}^k p^*_iv_i-x+c'\right\| = \left\|\Pi_W\left(\sum_{i=1}^k p^*_iv_i-x+c^*\right)\right\| < \left\|\sum_{i=1}^k p^*_iv_i-x+c^*\right\|,\] 
		which contradicts the optimality of $(p^*,c^*)$.
		
		Thus we have $\sum_{i=1}^k p^*_iv_i+c^*-x\in W$. Let $w = \sum_{i=1}^k p^*_iv_i+c^*-x$ and $G= \ch(\{v_1, \ldots, v_k\})+W^\circ+\{-x\}$. Then we have $w = \Pi_{G}(0)$ by the definition of $(p^*, c^*)$. Since $G$ is convex and $v_i-x\in G$ for all $i\in[k]$, by the Pythagorean theorem (Lemma~\ref{lem:pythagorean}) we have $w^\top (v_i-x-w) \ge 0$ for all $i\in[k]$, which implies $\|w\|^2\le \min_{i\in [k]}w^\top (v_i-x)\le \epsilon \|w\|$, i.e., $\|w\|\le \epsilon$.  Hence we have (B)$\implies$(A).
\end{proof}

Theorem~\ref{thm:projection} is now easy to prove using the above lemmas.

\begin{proof}[Proof of Theorem~\ref{thm:projection}]
	The upper bound on the number of iterations is proved in Lemma~\ref{lem:pad-running-time}.
	In the proof of Lemma~\ref{lem:pad-running-time}, we have shown $D_\varphi(x^*, z_{i+1}) \le D_\varphi(x^*, z_i)$ for all $x^* \in \K^*$ and $i$. This implies  $D_\varphi(x^*, y') = D_\varphi(x^*, z_{k+1}) \le D_\varphi(x^*, z_{k}) \le \cdots \le D_\varphi(x^*, z_{1}) = D_\varphi(x^*, y) $ for all $x^*\in\K^*$, which verifies the first property in Definition~\ref{def:projection_oracle}.
	The second property is a direct consequence of combing Lemmas~\ref{lem:cutting_plane_OMD} and \ref{lem:equivalence}.
\end{proof}

\subsection{The Oracle Complexity of Algorithm~\ref{alg:OMD}} \label{subsec:omd-number-calls}

\begin{theorem} \label{thm:omd-number-calls}
	Suppose $(\varphi,W)$ satisfies the PNIP property (Definition~\ref{def:pnip}), 
	$\epsilon\in(0, R]$,
	$\varphi$ is $\mu$-strongly convex and $A$ is an upper bound on $\max_{x^* \in \K^*} (\varphi(x^*)-\varphi(y_1))$.
	Then Algorithm~\ref{alg:OMD} only needs to call $\O_{\K, \K^*}$ for at most $\left\lceil 5d \log\frac{4R+2\sqrt{\frac2\mu \left(A + \left( \frac{\eta^2L^2}{2\mu}+ 2\eta LR+\epsilon\eta L \right)T \right)}}{\epsilon} \right\rceil$ times per round.
	
	In particular, setting  $\epsilon = \frac{R}{T}$ and $\eta = \frac1L\sqrt{\frac{2\mu A}{T}}$, Algorithm~\ref{alg:OMD} only needs to call $\O_{\K, \K^*}$ for at most  $\left\lceil 5d\log \left(\left(6\sqrt{T}+ \frac4R\sqrt{\frac{A}{\mu }}+4\right)T\right) \right\rceil$ times per round.
\end{theorem}
\begin{proof}
	According to Theorem~\ref{thm:projection}, round $t$ of Algorithm~\ref{alg:OMD} calls $\O_{\K, \K^*}$ for at most $$\left\lceil 5d \log\frac{4R+2\sqrt{\frac2\mu \min_{x^* \in \K^*}D_\varphi(x^*, y_t)}}{\epsilon} \right\rceil$$ times.
	Hence it suffices to obtain an upper bound on $\min_{x^* \in \K^*}D_\varphi(x^*, y_t)$.

    According to \eqref{eqn:omd-inproof-2} (substituting $T$ with $t$), we have:
    \begin{align*}
    \forall t\in[T],\
     \forall x^* \in \K^*: \quad  D_\varphi ( x^* , y_{t+1}) \le   D_\varphi ( x^* , y_1)  + \sum_{j=1}^{t} D_{\varphi}(x_j,y_{j+1}) - \eta\sum_{j=1}^{t} (f_j(\overline x_j) - f_j(x^*)) + \epsilon\eta Lt.
    \end{align*}
    Plug \eqref{eqn:omd-inproof-1.75} and \eqref{eqn:omd-inproof-3} into the above inequality, we have 
    \begin{align*}
    \forall t\in[T],\
    \forall x^* \in \K^*: \quad  
     D_\varphi ( x^* , y_{t+1}) &\le   D_\varphi ( x^* , y_1)  + \sum_{j=1}^{t} D_{\varphi}(x_j,y_{j+1}) - \eta\sum_{j=1}^{t} (f_j(\overline x_j) - f_j(x^*)) + \epsilon\eta Lt \\
     &\le   \varphi(x^*) - \varphi(y_1)  + \frac{\eta^2L^2}{2\mu}t - \eta\sum_{j=1}^{t} (f_j(\overline x_j) - f_j(x^*)) + \epsilon\eta Lt \\
     &\le   A  + \frac{\eta^2L^2}{2\mu}t + \eta\sum_{j=1}^{t} \norm{f_j} \cdot \norm{\overline x_j - x^*} + \epsilon\eta Lt \\
     &\le A + \frac{\eta^2L^2}{2\mu} T + 2\eta LRT + \epsilon\eta LT.
    \end{align*}
    For $t=1$ we also have $D_\varphi(x^*, y_1) \le A$.
    Therefore Algorithm~\ref{alg:OMD} calls $\O_{\K, \K^*}$ for at most $$\left\lceil 5d \log\frac{4R+2\sqrt{\frac2\mu \left(A + \left( \frac{\eta^2L^2}{2\mu}+ 2\eta LR+\epsilon\eta L \right)T \right)}}{\epsilon} \right\rceil$$ times per round.

    When $\epsilon = \frac{R}{T}$ and $\eta = \frac1L\sqrt{\frac{2\mu A}{T}}$, we have
    \begin{align*}
    \frac2\mu \left(A + \left( \frac{\eta^2L^2}{2\mu}+ 2\eta LR+\epsilon\eta L \right)T \right)
    \le \frac{4A}{\mu} + 6R\sqrt{\frac{2 AT}{\mu}}
   \le \left(2\sqrt{\frac{A}{\mu}}+3R\sqrt{T}\right)^2,
    \end{align*}
    so the number of oracle calls per iteration is at most $\left\lceil 5d\log \left(\left(6\sqrt{T}+ \frac4R\sqrt{\frac{A}{\mu }}+4\right)T\right) \right\rceil$.
\end{proof}

%\subsection{$\alpha$-approximation algorithms for semi-bandit setting}

% \begin{theorem}
% Assuming $\K\subset \{0,1\}^d$, and if $x\in \K$ then $\forall y\succeq x$, $y\in \K$. For any learning rate $\eta>0$, tolerance $\epsilon>0$ and any loss sequence $\{f_t\}_{t=1}^T\in (([0,1])^d)^T$, setting $\varphi(x)=\sum_{i=1}^dx_i(\ln x_i-1)$, Algorithm \ref{alg:semi_bandit} has the following regret guarantee: $\forall x^*\in \K$,

% \[\E[\alpha-\reg_T]=\E[\sum_{t=1}^Tf_t^\top(x_t-\alpha^* x^*)]=\frac{d(1+\alpha\ln\alpha)}{\eta}+ \eta d\alpha\]
% \end{theorem}
% \begin{algorithm}
% \caption{Online stochastic mirror descent with negative entropy for semi-bandit}
% \label{alg:semi_bandit}
% \begin{algorithmic}[1]
% 	\STATE \textbf{Input:} Learning rate $\eta>0$, tolerance $\epsilon$, $\alpha$- approximation oracle $\O_{\K,\alpha\K}$.
% 	\STATE $y_1=0$.
% 	\FOR{$t=1$ to $T$}
%     	\STATE Let $(x_{t},\bv_{t},\bp_{t})$ be the output of $\widetilde{\Pi}_\K (y_{t},\epsilon) $,
% 		\STATE play $v^t_i$ according to distribution $\bp^t$, observe $f_t$.
% 		\STATE update $\nabla \varphi(y_{t+1}) \leftarrow \nabla \R (x_t) - \eta f_t $ 
% 	\ENDFOR
% \end{algorithmic}
% \end{algorithm}
% \zhiyuan{To be continued}

\section{Efficient Online Improper Linear Optimization via Continuous Multiplicative Weight Update (CMWU)}
\label{sec:cmwu}

In this section, we design our second online improper linear optimization algorithm (in the full information setting) based on the continuous multiplicative weight update (CMWU) method.
Compared with Algorithm~\ref{alg:OMD}, the CMWU-based algorithm allows loss vectors to come from a general convex cone $W$ and does not require the PNIP condition (Definition~\ref{def:pnip}).

%We now use the separation-or-decomposition oracle as a building block to design our second online improper linear optimization algorithm (in the full information setting) based on the continuous multiplicative weight update (CMWU) method.
%Compared with Algorithm~\ref{alg:OMD}, this CMWU-based algorithm allows loss vectors to come from a general convex cone $W$ and does not require the PNI condition.

\subsection{Separation-or-Decomposition Oracle}\label{subsec:SOD}

We first construct a \emph{separation-or-decomposition (SOD)} oracle (Algorithm~\ref{alg:SOD}) using $\O_{\K, \K^*}$, which we will use to design the online improper linear optimization algorithm later in this section.
Given a point $x \in B(0, R)$, the SOD oracle either outputs a \emph{separating hyperplane} between $x$ and $\K^*$, or outputs a \emph{distribution of points in $\K$} which approximately dominates $x$ in every direction in $W$.
The guarantee of the SOD oracle is summarized in Theorem~\ref{thm:SOD}.

% \begin{lemma}
% Let $\bp$ be the optimal of $\D(x,\bv,W)$, then $\forall w\in W,|w|\leq L$, $\sum\limits_{i=1}^k p_i w^Tv_i\leq w^Tx+L\epsilon$.
% \end{lemma}

\begin{algorithm}[htbp]
\caption{Separation-or-Decomposition Oracle, $\mathcal{SOD}(x,\epsilon,W)$} \label{alg:SOD}
\begin{algorithmic}[1]
	\REQUIRE Point $x\in B(0, R)$, tolerance $\epsilon>0$, convex cone $W \subseteq \reals^d$
	\ENSURE Decomposition $V = (v_1,\ldots,v_k) \in \reals^{d\times k}$, $p = (p_1,\ldots,p_k)^\top \in \Delta^{k-1}$, such that $v_i \in \K$ ($\forall i\in[k]$) and $\norm{\sum_{i=1}^kp_iv_i - x + c} \le 3\epsilon$ for some $c\in W^\circ$. %Decomposition of $x$: $(v_1,\ldots,v_k)$, $(p_1,\ldots,p_k)$, where $\bp$ is a distribution on $\bv$ s.t. $\sum_{i=1}^kp_iv_i-x\in -W^\circ+B(0,3\epsilon)$.
    \\{\textbf{Or:}} Separating hyperplane $(w,b) \in \reals^d\times \reals$, such that $\norm{w}=1$ and $w^\top x\le b-\epsilon \le \min_{x^*\in K^*}w^\top x^*-\epsilon.$
    \STATE $k\gets \left\lceil 5d\log \frac{4R}{\epsilon} \right\rceil$
    \STATE $W_1 \gets W\cap B(0,1)$
	\FOR{$i=1$ \TO $k$}
		\STATE $w_i \gets \frac{\int_{W_i}w \dw}{\vol{W_i}}$ %\frac{\int_{w\in W_t}wdw}{Vol(W_t)}$.  
		\STATE $v_i \gets \O_{\K,\K^*}(w_i)$
        \IF{$w_i^\top x \leq w_i^\top v_i-\epsilon$}
        	\RETURN{ Separating hyperplane $\left( \frac{w_i}{\norm{w_i}}, \frac{w_i^\top v_i}{\norm{w_i}} \right)$}
        \ELSE
        	\STATE  $W_{i+1} \gets W_i \cap \{w\in \reals^d: w^\top(v_i-x) \ge \epsilon\}$
        \ENDIF
	\ENDFOR
    \STATE Solve $\min\limits_{p\in \Delta^{k-1}, c\in W^\circ} \norm{\sum_{i=1}^k p_iv_i+c-x}$ to get $p$
    \RETURN $V = (v_1, \ldots, v_k), p$
\end{algorithmic}
\end{algorithm}

\begin{theorem}\label{thm:SOD}
	For any $x \in B(0, R)$ and $\epsilon \in (0, 2R]$,
the separation-or-decomposition oracle in Algorithm~\ref{alg:SOD}, $\mathcal{SOD}(x,\epsilon,W)$, returns one of the two followings, using at most $k = \left\lceil 5d\log{\frac{4R}{\epsilon}} \right\rceil$ calls of $\O_{\K,\K^*}$:
\begin{enumerate}
\item a decomposition $V = (v_1,\ldots,v_k) \in \reals^{d\times k}$, $p = (p_1,\ldots,p_k)^\top \in \Delta^{k-1}$, such that $v_i \in \K$ ($\forall i\in[k]$) and $\norm{\sum_{i=1}^kp_iv_i - x + c} \le 3\epsilon$ for some $c\in W^\circ$.
%$\sum_{i=1}^kp_iv_i-x\in -W^\circ+B(0,3\epsilon)$;
\item a separating hyperplane $(w,b) \in \reals^d \times \reals$, where $\|w\|=1$ and $w^\top x\le b-\epsilon\le \min_{x^* \in \K^*}w^\top x^*-\epsilon$.
\end{enumerate}
\end{theorem}

The proof of Theorem~\ref{thm:SOD} is postponed to Appendix~\ref{app:proof-cmwu}.

\subsection{CMWU with Refining Domains}

%Before presenting our main result in this section, we first 
Now we look at a general online learning setting where the feasible domain is shrinking over time while being a superset of the target domain.
%First, we show that CMWU achieves low regret if the feasible domain is shrinking while staying a superset of the target domain.
Namely, suppose $\K^*$ is the target domain and $\K_t$ is the feasible domain in the $t$-th round. We assume $B(0, R) \supseteq \K_0\supseteq \K_1\supseteq  \K_2 \dots  \supseteq \K_T\supseteq (1-\gamma)\K^*+\gamma \K_0$ for some $\gamma \in (0, 1]$. In round $t$, the player only knows $\K_1, \ldots, \K_t$  and does not know $\K_{j}$ for all $j>t$. % be a small number we will define later.
We can still run CMWU in this setting, using the knowledge of $\K_t$ at iteration $t$ - the algorithm is given in Algorithm~\ref{alg:CMWU}.
Theorem~\ref{thm:CMWU_general} bounds the regret of Algorithm~\ref{alg:CMWU} in this setting.

\begin{algorithm}[htbp]
\caption{Continuous Multiplicative Weight Update (CMWU) with Refining Domains}\label{alg:CMWU}
\begin{algorithmic}[1]
	\REQUIRE Learning rate $\eta>0$, time horizon $T \in \mathbb N_+$
	\FOR{$t=1$ \TO $T$}
    	\STATE Receive current domain $\K_t$
        %\STATE Define $z_t(x) = \exp(-\eta\sum\limits_{i=1}^{t-1}f_i(x))$ and $Z_t= \int_{x\in K_t}\exp(-\eta\sum\limits_{i=1}^{t-1}f_i(x))\dx$.
        \STATE Play $x_t = \frac{\int_{\K_t} e^{ -\eta\sum_{i=1}^{t-1}f_i(x) } x \dx}{\int_{ \K_t} e^{ -\eta\sum_{i=1}^{t-1}f_i(x) } \dx}$ \label{line:cmwu-refine-mean-compute}
		\STATE Receive loss vector $f_t$
	\ENDFOR
\end{algorithmic}
\end{algorithm}

\begin{theorem}\label{thm:CMWU_general}
%For every $d, \gamma, \eta, T > 0$, for every loss taking values in [-1,1], and every refining convex set series 
Suppose $B(0, R) \supseteq \K_0\supseteq \K_1\supseteq \K_2\supseteq \cdots \supseteq \K_T \supseteq (1-\gamma)\K^*+\gamma \K_0$ for $\gamma\in(0, 1]$.
Then for any $0<\eta\le \frac{1}{LR} $, Algorithm~\ref{alg:CMWU} has the following regret guarantee:
\begin{equation*}
\forall x^*\in \K^*: \quad \sum_{t=1}^T (f_t(x_t)-f_t(x^*)) \le \frac{d\log \frac{1}{\gamma}}{\eta} + \eta L^2R^2T + \gamma LRT- \frac{\sum_{t=1}^{T}\delta_t}{\eta},
\end{equation*}
where 
\begin{equation*}
\delta_t :%= \log \frac{Z'_t}{Z_t} 
=\log\frac{\int_{\K_{t-1}} e^{-\eta \sum_{i=1}^{t-1}f_i(x)} \dx}{ \int_{\K_{t}} e^{-\eta \sum_{i=1}^{t-1}f_i(x)} \dx}.
\end{equation*}

 In particular, setting $\gamma = 1/T$ and $\eta = \frac{1}{LR} \min\left\{1, \sqrt{\frac{d\log T}{T}} \right\}$, we have 
 \begin{equation*}
 \forall x^*\in \K^*: \quad  \sum_{t=1}^T (f_t(x_t)-f_t(x^*)) \le  LR\left(1+ 2\max\left\{\sqrt{dT \log T}, d\log T \right\}\right).
 \end{equation*}
\end{theorem}

\begin{proof}%{\cref{thm:CMWU_general}}
	We fix any $x^* \in \K^*$ and
	denote $\bar \K  := (1-\gamma)x^*+\gamma \K_0$. Since $(1-\gamma)\K^*+\gamma \K_0\subseteq \K_T$, we have $\bar \K \subseteq \K_T$. We define
	\begin{equation*}
	z_t(x) := e^{-\eta \sum_{i=1}^{t-1}f_i(x)}, \quad
	Z_t := \int_{\K_{t}}z_t(x)\dx,\quad Z'_t := \int_{\K_{t-1}}z_t(x)\dx.
	\end{equation*}
	A straightforward calculation gives us:  
	\begin{align*}
	\log \frac{Z'_{T+1}}{Z'_1} 
	&=\log \left(\frac{\int_{\K_{T}}e^{-\eta \sum_{t=1}^T f_t(x)}\dx}{\int_{\K_{0}}1 \dx}\right)\\
	&\ge \log\left(\frac{\int_{\bar \K}e^{-\eta \sum_{t=1}^T f_t(x)}\dx}{\int_{\K_{0}}1\dx}\right)\\
	&= \log \left(\frac{\int_{\K_0} e^{-\eta \sum_{t=1}^T f_t((1-\gamma)x^*+\gamma x)}\gamma^d \dx}{\int_{\K_{0}}1\dx}\right)  \\
	&= \log \left(\frac{\int_{\K_0}e^{-\eta \sum_{t=1}^T ((1-\gamma)f_t(x^*)+\gamma f_t(x))}\gamma^d \dx}{\int_{\K_{0}}1\dx}\right)\\
	&\ge \log \left(\frac{\int_{\K_0}e^{-\eta \sum_{t=1}^T (f_t(x^*)+ \gamma LR)}\gamma^d \dx}{\int_{ \K_{0}}1\dx}\right)\\
	&= d\log \gamma -\eta \sum_{t=1}^T f_t(x^*)- \eta \gamma LRT.
	\end{align*}
	On the other hand, we have
	\begin{align*}
	\log \frac{Z'_{t+1}}{Z_t} 
	=& \log \left( \int_{\K_t}\frac{z_t(x)}{Z_t} e^{-\eta f_t(x)} \dx \right)\\
	\le & \log \left( \int_{\K_t} \frac{z_t(x)}{Z_t} \left(1-\eta f_t(x)+ (\eta f_t(x))^2\right) \dx\right)  \\
	\le & \left( \int_{\K_t} \frac{z_t(x)}{Z_t} \left(1-\eta f_t(x)+ (\eta f_t(x))^2\right) \dx \right) -1 \\
	\le & \int_{\K_t} \frac{z_t(x)}{Z_t} \left(-\eta f_t(x)+ \eta^2L^2R^2\right) dx \\
	= &-\eta f_t \left(\int_{\K_t} \frac{z_t(x)}{Z_t}x \dx \right)+ \eta^2L^2R^2 \\
	= & -\eta f_t(x_t)+\eta^2L^2R^2,
	\end{align*}
	where the first inequality is due to $e^a \le 1+a+a^2$ ($\forall a\le 1$) and $|\eta f_t(x)| \le \eta LR \le 1$, the second inequality is due to $\log a \le a-1$ ($\forall a>0$), and the third inequality is due to $|\eta f_t(x)| \le \eta LR$.

	Note that $\delta_t = \log \frac{Z'_t}{Z_t}$.
%	\begin{equation*}
%	\delta_t = \log \frac{Z'_t}{Z_t} =\log\frac{\int_{x\in K_{t-1}}\exp(-\eta \sum_{i=1}^{t-1}f_i(x))dx}{ \int_{x\in K_{t}}\exp(-\eta \sum_{i=1}^{t-1}f_i(x))dx}.
%	\end{equation*}
	Combining the two bounds above, we get:  
	\begin{align*}
	\sum_{t=1}^{T} \left(-\eta f_t(x_t)+\eta^2L^2R^2\right) 
	&\ge  \sum_{t=1}^T \log\frac{Z'_{t+1}}{Z_t}
	=\log\frac{Z'_{T+1}}{Z'_1}+\sum_{t=1}^{T} \log\frac{Z'_{t}}{Z_t}\\
	&\ge d\log \gamma -\eta \sum_{t=1}^T f_t(x^*)- \eta \gamma LRT + \sum_{t=1}^T \delta_t.
	\end{align*}
	In other words,
	\begin{equation*}
	%\Reg_{\K,\K^*} (T) 
	 \sum_{t=1}^T (f_t(x_t)-f_t(x^*)) \le \frac{d\log \frac{1}{\gamma}}{\eta} + \eta L^2R^2T + \gamma LRT-\frac{\sum_{t=1}^T \delta_t}{\eta}. 
	\end{equation*}
	
	The second regret bound in the theorem follows directly by plugging in the values of $\gamma$ and $\eta$ stated in the theorem.
	Note that $\delta_t\ge 0$ for all $t\in[T]$
	since $\K_{t}\subseteq \K_{t-1}$.
\end{proof}

\subsection{Online Improper Linear Optimization via CMWU}

Now we are ready to present our CMWU-based online improper learning algorithm.
At a high level, this algorithm is a specialized implementation of Algorithm~\ref{alg:CMWU} for the online improper linear optimization problem.
The algorithm starts with an initial convex domain $\K_0$ which is a superset of $\K^*$, and maintains a convex domain $\K_t$ at iteration $t$.
In iteration $t$, the algorithm first computes the mean $x_t$ of a log-linear distribution over $\K_t$ using random walk, as in Line~\ref{line:cmwu-refine-mean-compute} of Algorithm~\ref{alg:CMWU}. Then the algorithm calls the SOD oracle on $x_t$. If the SOD oracle returns a distribution of points in $\K$, then we can play according to that distribution, since the SOD oracle ensures that the expected loss of this distribution is not much larger than that of $x_t$.
If the SOD oracle returns a separating hyperplane between $x_t$ and $\K^*$, then the algorithm replaces $\K_t$ with the intersection of the original $\K_t$ and the half-space given by this hyperplane that contains $\K^*$, and repeats the same process for the new $\K_t$ until a decomposition is returned by the SOD oracle.
Note that each time $\K_t$ is updated, the mean $x_t$ of the log-linear distribution is \emph{not} in the new $\K_t$, which according to Lemma~\ref{lem:center-cut-to-use} implies that a constant probability mass is removed. This allows us to bound the total number of oracle calls.
We detail our algorithm in Algorithm~\ref{alg:improper_CMWU} and its regret bound in Theorem~\ref{thm:improper_CMWU}.

\begin{algorithm}[htbp]
\caption{CMWU for Online Improper Linear Optimization}\label{alg:improper_CMWU}
\begin{algorithmic}[1]
	\REQUIRE Learning rate $\eta>0$, tolerance $\gamma>0$, initial convex domain $\K_0$, convex cone $W$, time horizon $T \in \mathbb N_+$
   % \STATE Construct separation or decomposition oracle $\mathcal{SOD}$ by $\O_{\K,\K^*}$.
	\STATE $\K_1 \gets \K_0$
	\FOR{$t=1$ \TO $T$}
    	%\STATE  $z_t(x) = \exp(-\eta\sum\limits_{i=1}^{t-1}f_i(x))$ and $Z_t= \int_{x\in \K_t}\exp(-\eta\sum\limits_{i=1}^{t-1}f_i(x))dx$.
        \STATE  $x_t \gets \frac{\int_{\K_t} e^{ -\eta\sum_{i=1}^{t-1}f_i(x) } x \dx}{\int_{ \K_t} e^{ -\eta\sum_{i=1}^{t-1}f_i(x) } \dx}$ %by random walk
        \WHILE{$\mathcal{SOD}(x_t,2\gamma R,W)$ returns a separating hyperplane $(w, b)\in \reals^d\times \reals$}
        	%\STATE Let $(w,b)$ be the output of $\S(x_t,\gamma R)$, assuming $\|w\|=1$.
        	\STATE $\K_t \gets \K_t\cap \{x\in\reals^d: w^\top x\ge b-2\gamma R\}$
			\STATE  $x_t \gets \frac{\int_{\K_t} e^{ -\eta\sum_{i=1}^{t-1}f_i(x) } x \dx}{\int_{ \K_t} e^{ -\eta\sum_{i=1}^{t-1}f_i(x) } \dx}$ %by random walk
        \ENDWHILE
        \STATE Let $(V,p) \in \reals^{d\times k} \times \Delta^{k-1}$ be the output of $\mathcal{SOD}(x_t,2\gamma R,W)$
        \STATE Play $\tilde x_t = v_i$ with probability $p_i$ ($i=1, \ldots, k$), where $V = (v_1, \ldots, v_k)$
        \STATE  $\K_{t+1}\gets \K_t$
		\STATE Receive loss vector $f_t$
	\ENDFOR
\end{algorithmic}
\end{algorithm}

\begin{theorem}\label{thm:improper_CMWU}
%For every $d, \gamma, \eta, T, L > 0$, for every linear loss $f_t\in W\cap B(0,L)$  and 
Suppose that the initial convex domain $\K_0$ satisfies $ \K^*\subseteq \K_0\subseteq B(0,R)$.
%Denote by $\tilde x_t$ the point played by Algorithm~\ref{alg:improper_CMWU} in round $t$.
Then for any $\gamma \in (0, 1]$ and $\eta \in \left(0, \frac{1}{LR}\right]$, Algorithm~\ref{alg:improper_CMWU} satisfies the following regret guarantee: 
\begin{equation*}
\forall x^* \in \K^*: \quad  \E\left[\sum_{t=1}^T ( f_t(\tilde x_t)-f_t(x^*) ) \right] \le \frac{d\log \frac{1}{\gamma}}{\eta} + \eta L^2R^2T + 7\gamma LRT- \frac{s}{5\eta},
\end{equation*}
where $s = \sum_{t=1}^T s_t$, and $s_t$ is the number of separating hyperplanes returned by the SOD oracle during round $t$.

In particular, if we set $\gamma = \frac{1}{T}, \eta = \frac{1}{LR} \min\left\{1, \sqrt{\frac{d\log T}{T}} \right\}$,  then we have 
\begin{equation*}
\forall x^* \in \K^*: \quad  \E\left[\sum_{t=1}^T ( f_t(\tilde x_t)-f_t(x^*) ) \right] \le LR\left(7+ 2\max\left\{\sqrt{dT \log T}, d\log T \right\}\right),
\end{equation*}
and in this case Algorithm~\ref{alg:improper_CMWU} calls $\O_{\K,\K^*}$ for $O\left(dT\log T\right)$ times in $T$ rounds.
%The oracle complexity is $O(Td\log T)$ totally. Moreover, $\Reg_{\K,\K^*}(T)\!\!\leq\! 0$ when $k\!\!\geq\!\! 40d\sqrt{\log T}+\!1$.
\end{theorem}

\begin{proof} %{\cref{thm:improper_CMWU}}
	In the proof, we use $\bar \K_t$ and $\bar x_t$ to denote the values of $\K_t$ and $x_t$ \emph{at the end of iteration $t$} ($\bar \K_0 = \K_0$).
We define
\begin{equation*}
z_t(x) := e^{-\eta \sum_{i=1}^{t-1}f_i(x)}, \quad
Z_t := \int_{\bar\K_{t}}z_t(x)\dx, \quad
Z'_t := \int_{\bar\K_{t-1}}z_t(x)\dx, \quad
\delta_t := \log \frac{Z_t'}{Z_t}.
\end{equation*}	
	
	We first prove the following two claims: 
	\begin{enumerate}[(i)]
		\item \label{item:cmwu-inproof-claim-1} For all $t\in\{0, 1, \ldots, T\}$, we have % $\K_t$ maintained by the algorithm always satisfies
		$(1-\gamma)\K^*+\gamma \bar \K_0 %\subseteq (1-\gamma)\K^*+\gamma B(0,R)
		\subseteq \bar\K_t$.
		\item \label{item:cmwu-inproof-claim-2}  For all $t\in[T]$, we have $\delta_t\ge \frac{s_t}{5}$.
		% Let $B'_t = B_t \cap \{x|w^Tx\geq d-\gamma R\}$, $\delta_$
	\end{enumerate}
	
	We use induction to prove \eqref{item:cmwu-inproof-claim-1}. It holds for $t=0$ since $\K^*\subseteq\K_{0}=\bar \K_0$ and $\K_0$ is convex. Suppose it holds for $t-1$. If $\bar\K_t = \bar\K_{t-1}$, then it already holds for $t$.
	Otherwise, consider any separating hyperplane $(w,b) \in \reals^d\times\reals$ obtained in round $t$, which is the output of $\mathcal{SOD}(x', 2\gamma R,W)$ for some $x'$. By the guarantee of the SOD oracle, we have 
	\[ w^\top x' \le b-2\gamma R \le \min_{x^*\in\K^*} w^\top x^* -2\gamma R.\]
	This implies
	\[ (1-\gamma)\K^*+ \gamma \bar\K_0 \subseteq  (1-\gamma)\K^*+ \gamma B(0, R) \subseteq \K^*+ B(0,2\gamma R)\subseteq \{x\in\reals^d: w^\top x\ge b-2\gamma R\}.\]
	Note that $\{x\in\reals^d: w^\top x\ge b-2\gamma R\}$ is exactly the half-space to intersect with when updating $\K_t$.
	Hence we know that during the execution of the algorithm, $\K_t$ is always a superset of $(1-\gamma)\K^*+ \gamma \bar\K_0$.
	This proves \eqref{item:cmwu-inproof-claim-1}.
	
	For \eqref{item:cmwu-inproof-claim-2}, note that each time $\K_t$ is updated, the mean of the distribution over $\K_t$ with density proportional to $z_t(x)$ is not included in the interior of the new $\K_t$. By Lemma~\ref{lem:center-cut-to-use}, this implies $\int_{\text{new } \K_t}z_t(x)\dx \le \left(1 - \frac{1}{2e}\right) \int_{\text{old } \K_t}z_t(x)\dx$.
	Hence we have $-\delta_t = \log \frac{Z_t}{Z_t'} \le \log \left(1 - \frac{1}{2e}\right)^{s_t}$, which gives $\delta_t\ge \frac{s_t}{5}$.

%	On the other hands, since $x_t \notin \{x|w^Tx\geq b-\gamma R\}$, by Lemma~\ref{lem:center_cut}, we have 
%	\[\int_{ \K_t \cap \{x|w^Tx\geq b-\gamma R\}}z_t(x)dx\leq (1-\frac{1}{2e})\int_{ \K_t} z_t(x)dx.\]
%	
%	In other words, each time we cut $\K_t$ with a hyperplane, the integral of log linear function, $\int_{\K_t}z_t(x)\dx$ is shrunk by a constant. Thus, $\delta_t\geq s_t \ln \frac{1}{1-\frac{1}{2e}}=s_t \ln \frac{2e}{2e-1}\geq \frac{k_t}{5}$.

	Now we show the regret bound.
	From \eqref{item:cmwu-inproof-claim-1} we know
	\[(1-\gamma)\K^*+\gamma \bar \K_0 \subseteq \bar\K_T \subseteq \bar\K_{T-1}\subseteq \cdots\subseteq \bar\K_1 \subseteq \bar\K_0 \subseteq B(0, R).\]
	Therefore, we can apply Theorem~\ref{thm:CMWU_general} to get
	\begin{align*}
	\forall x^*\in \K^*: \quad \sum_{t=1}^T (f_t(\bar x_t)-f_t(x^*)) &\le \frac{d\log \frac{1}{\gamma}}{\eta} + \eta L^2R^2T + \gamma LRT- \frac{\sum_{t=1}^{T}\delta_t}{\eta} \\
	&\le \frac{d\log \frac{1}{\gamma}}{\eta} + \eta L^2R^2T + \gamma LRT- \frac{s}{5\eta},
	\end{align*}
	where the second inequality is due to \eqref{item:cmwu-inproof-claim-2}.
	
	The actual algorithm does not play $\bar x_t$, but a random $\tilde x_t$.
	Namely, letting $(V,p) \in \reals^{d\times k} \times \Delta^{k-1}$ be the output of $\mathcal{SOD}(\bar x_t,2\gamma R,W)$, we have that $\tilde x_t$ is equal to
	 $v_i$ with probability $p_i$ ($i=1, \ldots, k$), where $V = (v_1, \ldots, v_k)$.
	By Theorem~\ref{thm:SOD} we know that there exists $ c\in W^\circ$ such that $\norm{\sum_{i=1}^k p_iv_i +c-\bar x_t }\le  6\gamma R$, which implies (note $f_t\in W \cap B(0, L)$)
	\[
	\E[f_t(\tilde x_t)]=f_t\left(\sum_{i=1}^k p_iv_i\right)
	\le f_t\left(\sum_{i=1}^k p_iv_i +c\right) \le f_t(\bar x_t)+6\gamma LR.
	\] 
	Therefore we have
	\begin{align*}
	\forall x^*\in\K^*: \E\left[\sum_{t=1}^T ( f_t(\tilde x_t)-f_t(x^*) ) \right]
	 =\,& \E\left[\sum_{t=1}^T ( f_t(\tilde x_t)-f_t(\bar x_t) ) \right] +\sum_{t=1}^T(f_t(\bar x_t)-f_t(x^*))\\
	 \le\, & 6\gamma LRT + \frac{d\log \frac{1}{\gamma}}{\eta} + \eta L^2R^2T + \gamma LRT- \frac{s}{5\eta}\\
	=\, & \frac{d\log \frac{1}{\gamma}}{\eta} + \eta L^2R^2T + 7\gamma LRT- \frac{s}{5\eta}.
	\end{align*}

	Setting  $\gamma = \frac{1}{T}$ and $\eta = \frac{1}{LR} \min\left\{1, \sqrt{\frac{d\log T}{T}} \right\}$, the above bound becomes
	\begin{align*}
	\forall x^*\in\K^*: \E\left[\sum_{t=1}^T ( f_t(\tilde x_t)-f_t(x^*) ) \right]
	\le LR\left(7+ 2\max\left\{\sqrt{dT \log T}, d\log T \right\}\left( 1 - \frac{s}{10d\log T}  \right) \right).
	\end{align*}
	We can use the above regret bound to bound the number of oracle calls in Algorithm~\ref{alg:improper_CMWU}.
	Since the regret is always lower bounded by $-2LRT$, the above regret upper bound implies $s = O\left( T \right)$.
	Therefore Algorithm~\ref{alg:improper_CMWU} calls the SOD oracle for $s+T= O\left( T \right)$ times.
	Note that
	each implementation of SOD needs to call $\O_{\K, \K^*}$ for $O\left( d\log \frac{4R}{2\gamma R} \right) = O(d\log T)$ times (Theorem~\ref{thm:SOD}). 
	We conclude that the total number of calls to $\O_{\K, \K^*}$ in Algorithm~\ref{alg:improper_CMWU} is at most $O\left(dT\log T\right)$.
	%For the oracle complexity, we only need to show Algorithm \ref{alg:improper_CMWU} calls $O(T) \mathcal{SOD}$ oracles, since this is the only way to access $\O_{\K,\K^*}$. Note that the regret is lower bounded by $-LRT$, which implies $k=O(\sqrt{Td\ln T})$, the proof is concluded since $\mathcal{SOD}$ will return one decomposition per iteration and $O(\sqrt{Td\ln T})$ separation hyperplanes in total.
\end{proof}

\begin{remark*}
	Intuitively, when the SOD oracle is called in Algorithm~\ref{alg:improper_CMWU}, between the two outcomes (separation and decomposition) we should prefer decomposition, since this means we can make the play and move on to the next iteration.
	However, Theorem~\ref{thm:improper_CMWU} shows an interesting {trade-off between oracle complexity and regret}: the more oracle calls, the less the regret. This means obtaining separating hyperplanes helps the regret.
	Interestingly, we obtain our upper bound on the oracle calls by observing that regret can never be lower by $-2LRT$.
\end{remark*}

%Key to our analysis is the observation that when we call the SOD oracle,
%among the two outcomes (separation and decomposition), we would prefer decomposition, since this means we can play a distribution with desired mean.  Therefore, towards an efficient algorithm, the key is to bound the number of the calls of the oracle when a separation hyperplane is returned. \cite{garber2017efficient} used a distance potential argument, that is, each time we get a separation hyperplane, we project $x_t$ to that hyperplane, which would decrease a distance potential by $O(1/T)$. On the other hand, this potential could increase by $1/\sqrt{T}$ per round. As a conclusion, the total number of calls for which a separation hyperplane is returned is at most $O(T^{\frac{3}{2}})$ in $T$ rounds. Different from Garber's  OGD algorithm with a distance potential argument, Algorithm~\ref{alg:improper_CMWU} exhibits a trade-off between the oracle complexity and the regret: each time  we call the SOD oracle and receive a separation hyperplane, we can decrease the regret bound by $\frac{1}{\eta}$, with $\eta =\widetilde{\Theta}(\sqrt{\frac{d}{T}})$, via refining the outer approximated polytope of $\K^*$. Since regret is at least $-T$, we will have at most $O(\sqrt{Td})$ oracles calls for which a separation hyperplane is returned.

\section{$\alpha$-Regret Minimization in the Bandit Setting}\label{sec:bandit}

In this section we consider the $\alpha$-regret minimization problem in the bandit setting, where $W = \reals_+^d$, $\K \subseteq \reals_+^d \cap B(0, R)$ and $\K^* = \alpha \K$.
Similar to \citep{kakade2009playing}, we assume we know a \emph{$\beta$-barycentric spanner} for $\K$.
This concept was first introduced by \cite{awerbuch2004adaptive}.

\begin{definition}[Barycentric spanner]
	A set of $d$ linearly independent vectors $\{q_1,\ldots,q_d\}\subset\reals^d$ is a $\beta$-barycentric spanner for a set $\K\subset \reals^d$, denoted by $\beta$-BS$(\K)$, if $\{q_1,\ldots,q_d\}\subseteq \K$ and for all $x\in\K$, there exist $\beta_1, \ldots, \beta_d \in[-\beta,\beta]$ such that $x= \sum_{i=1}^d \beta_iq_i$.
\end{definition}

Given $\{q_1,\ldots,q_d\}$ which is a $\beta$-BS$(\K)$,
define $Q := \sum_{i=1}^d q_iq_i^\top$ and
$M := (q_1, \ldots, q_d) \in \reals^{d\times d}$.
Then we have $Q = MM^\top$ and $Me_i = q_i, M^{-1}q_i = e_i$ ($\forall i\in[d]$), where $e_i$ is the $i$-th standard unit vector in $\reals^d$.

\paragraph{The need for a new regularization.}  The bandit algorithm of \cite{garber2017efficient} additionally requires a certain boundedness property of barycentric spanners, namely:
$$ \max_{i \in [d]}  q_i ^\top  Q^{-2} q_i \le \chi .$$
%where $Q := \sum_{i=1}^d q_iq_i^\top$.
However, for certain bounded sets this quantity may be unbounded, such as the two-dimensional axis-aligned rectangle with one axis being of size unity, and the other arbitrarily small. This unboundedness creates problems with the unbiased estimator of loss vector, whose variance can be as large as certain geometric properties of the decision set. To circumvent this issue, we design a new regularizer called \emph{barycentric regularizer}, which gives rise to an unbiased estimator coupled with an online mirror descent variant that automatically ensures constant variance. 

%The constant $\beta$ in the definition below can be taken to be one without loss of generality, see \citep{awerbuch2004adaptive}. 

Similar to \citep{kakade2009playing, garber2017efficient}, our bandit algorithm also simulates the full information algorithm with estimated loss vectors.
Namely, our algorithm implements Algorithm~\ref{alg:OMD} with a specific \emph{barycentric regularizer} $\varphi(x) = \frac12 x^\top Q^{-1}x$.
The algorithm is detailed in Algorithm~\ref{alg:bandit}, and its regret guarantee is given in Theorem~\ref{thm:bandit_reg}.

\begin{algorithm}[htbp]
\caption{Online Stochastic Mirror Descent with Barycentric Regularization }
\label{alg:bandit}
\begin{algorithmic}[1]
\REQUIRE Learning rate $\eta>0$, tolerance $\epsilon>0$, $\{q_1,\ldots, q_d\}$ - a $\beta$-BS($\K$) for some $\beta>0$, exploration probability $\gamma\in(0,1)$, time horizon $T \in \mathbb N_+$
\STATE Instantiate Algorithm \ref{alg:OMD} with parameters $\eta$, $\epsilon$, $\varphi(x)=\frac{1}{2}x^\top Q^{-1}x$, $W' =(M^\top)^{-1}\reals^d_+$, and $T$
	\FOR{$t=1$ \TO $T$}
    	\STATE Receive $\tilde{x}_t$ (the point to play in round $t$) from Algorithm \ref{alg:OMD}
		\STATE $b_t\gets \begin{cases}
		\textrm{EXPLORE}, & \text{with probability } \gamma \\
		\textrm{EXPLOIT}, & \text{with probability } 1-\gamma
		\end{cases}$
		\IF{$b_t=\textrm{EXPLORE}$}
        	\STATE Sample $i_t\in [d]$ uniformly at random, and play $q_{i_t}$ 
            \STATE Receive loss $l_t =q_{i_t}^\top f_t$
            \STATE $\tilde{f_t} \gets \frac{d}{\gamma}l_t Q^{-1}q_{i_t}$
        \ELSE
        	\STATE Play $\tilde{x}_t$ and receive loss $l_t = \tilde{x}_t^\top f_t$
            \STATE $\tilde{f_t}\gets 0$
        \ENDIF
     	\STATE Feed $\tilde{f_t}$ to Algorithm \ref{alg:OMD} as the loss vector for round $t$ (Note that when $\tilde{f_t}=0$,
     	in the next round
     	 Algorithm \ref{alg:OMD} can simply play according to the distribution computed in this round without any oracle calls.)
 	\ENDFOR
\end{algorithmic}
\end{algorithm}

\begin{theorem}\label{thm:bandit_reg}
	
	Denote by $z_t$ the point played by Algorithm~\ref{alg:bandit} in round $t$.
	Then for any $\gamma \in (0, 1)$, $\epsilon \in (0,\alpha R]$ and $\eta>0$, Algorithm~\ref{alg:bandit} satisfies the following regret guarantee: 
\[
\forall x^*\in \K:\quad \E\left[\sum_{t=1}^T (  f_t(z_t)-\alpha  f_t(x^*) ) \right] \le \frac{\alpha^2\beta^2d}{2\eta} + \frac{\eta L^2R^2d^2}{2\gamma}T +  2\gamma \alpha LRT+ \epsilon LT ,
\]
and the expected total number of calls to the oracle $\O_{\K}^\alpha$ in $T$ rounds is at most  $$(1+\gamma T)\left( 1+ 5d \log\frac{4\alpha R+2R \sqrt{\alpha^2\beta^2d^2 + \frac{\eta^2L^2R^2d^3}{\gamma}T+6\eta\alpha LRdT}}{\epsilon} \right).$$

%\[5d(1+\gamma T)\log\frac{4R+2\sqrt{2 R^2d\left(\frac{\alpha^2\beta^2d}{2} + \frac{T\eta^2L^2R^2d^2}{2\gamma}+T\eta L(\epsilon+2\alpha R)\right)}}{\epsilon}.\]

In particular, setting $\eta=\frac{\alpha\beta^{4/3}}{LRT^{2/3}}$, $\epsilon=\frac{\alpha R}{T}$ and $\gamma = \frac{\beta^{2/3}d}{T^{1/3}}$ (assuming $T>\beta^2d^3$ so $\gamma<1$), we have 
\[
\forall x^*\in \K:\quad \E\left[\sum_{t=1}^T (  f_t(z_t)-\alpha  f_t(x^*) ) \right] \le  \alpha  LR \left(3d (\beta T)^{2/3}+1\right), 
\]
and the expected total number of oracle calls in $T$ rounds is at most
 $O\left(d^2(\beta T)^{2/3}\log T\right)$.
% $O\left(d^2(\beta T)^{2/3}\log\left(Td\alpha^2\beta^2(T+d)\right)\right)$.
\end{theorem}

\begin{proof} 
	Let $x_t$, $y_t$ and $\tilde{x}_t$ be the same $x_t$, $y_t$ and $\tilde{x}_t$ appearing in Algorithm~\ref{alg:OMD} during our implementation. We define $\overline x_t:=\E[\tilde x_t| y_t]$ similarly to the proof of Theorem~\ref{thm:OMD_regret}.
	It is easy to see that $(\varphi, W')$ satisfies the PNIP property (Definition~\ref{def:pnip}) and $\tilde f_t \in W'$ for all $t\in[T]$, where $W' = (M^\top)^{-1}\reals_+^d$.

    Note that $\nabla \varphi(x)=Q^{-1}x$, which implies
    \begin{align*}
    D_{\varphi}(x_t,y_{t+1}) &= \frac12 (x_t-y_{t+1})^\top Q^{-1} (x_t-y_{t+1}) 
    = \frac12 (\nabla \varphi(x_t)-\nabla \varphi(y_{t+1}))^\top Q (\nabla \varphi(x_t)-\nabla \varphi(y_{t+1}))\\
    & = \frac{\eta^2}{2} \tilde{f}_t^\top Q \tilde{f}_t.
    \end{align*}
    Using the regret bound \eqref{eqn:omd-inproof-1.5} in the proof of Theorem~\ref{thm:OMD_regret}, for any $x^*\in \K$ we have 
	\begin{equation} \label{eqn:bandit-inproof-1}
	\begin{aligned}
	\sum_{t=1}^T \left( \tilde f_t( x_t)-\alpha \tilde f_t(x^*) \right)
	&\le \frac{1}{\eta}\left(D_\varphi(\alpha x^*, y_1) - D_\varphi(\alpha x^*, y_{T+1}) +\sum_{t=1}^T D_{\varphi}(x_t,y_{t+1})\right) \\
	&\le \frac{1}{\eta}\left(\varphi(\alpha x^*) - \min_{y\in\reals^d}\varphi(y) +\sum_{t=1}^T D_{\varphi}(x_t,y_{t+1})\right) \\
	&= \frac1\eta  \varphi(\alpha x^*) + \frac{\eta}{2} \sum_{t=1}^T  \tilde f_t^\top Q \tilde f_t  .
	\end{aligned}
	\end{equation}
	%Define $M := (q_1, \ldots, q_d) \in \reals^{d\times d}$.
	%Then we have $Q = MM^\top$ and $Me_i = q_i, M^{-1}q_i = e_i$ ($\forall i\in[d]$), where $e_i$ is the $i$-th standard unit vector in $\reals^d$.
	Since $\{q_i\}_{i=1}^d$ is a $\beta$-BS($\K$), there exist $\beta_1, \ldots, \beta_d \in[-\beta,\beta]$ such that $x^*=\sum_{i=1}^d\beta_iq_i$.  Then we have
	\begin{equation*} %\label{eqn:bandit-inproof-3}
	\varphi(\alpha x^*)
	=\frac{1}{2}(\alpha x^*)^\top Q^{-1}(\alpha x^*)
	=\frac{\alpha^2}{2}\norm{M^{-1}x^*}^2
	=\frac{\alpha^2}{2}\left\|\sum_{i=1}^d\beta_iM^{-1}q_i\right\|^2
	=\frac{\alpha^2}{2}\left\|\sum_{i=1}^d\beta_ie_i\right\|^2
	\le \frac{\alpha^2\beta^2d}{2}.
	\end{equation*}
	We also have
	\begin{align*}
	\E\left[\tilde{f}_t^\top Q \widetilde{f_t}\right]
	&= \gamma \sum_{i=1}^d \frac1d \left(\frac{d}{\gamma} q_i^\top f_t \right)^2 q_i^\top Q^{-1} Q Q^{-1}q_i =
	\frac{d}{\gamma}\sum\limits_{i=1}^d \left(q_i^\top f_t \right)^2q_{i}^\top Q^{-1} q_{i} \\
	&\le \frac{L^2R^2d}{\gamma}\sum\limits_{i=1}^d q_{i}^\top (MM^\top)^{-1} q_{i}
	= \frac{L^2R^2d}{\gamma} \sum\limits_{i=1}^d e_i^\top e_i =\frac{L^2R^2d^2}{\gamma}.
	\end{align*}
	Hence by taking expectation on \eqref{eqn:bandit-inproof-1} we get
	\begin{equation}\label{eqn:bandit-inproof-2}
	\E\left[ \sum_{t=1}^T \left( \tilde f_t( x_t)-\alpha \tilde f_t(x^*) \right) \right]
	\le \frac{\alpha^2\beta^2d}{2\eta} + \frac{\eta L^2R^2d^2}{2\gamma}T.
	\end{equation}
    Note that $\E[\tilde f_t |  x_t] = \gamma \sum_{i=1}^d \frac1d \frac{d}{\gamma}  Q^{-1}q_iq_i^\top f_t = Q^{-1} \left(\sum_{i=1}^d q_iq_i^\top \right)f_t = f_t$.
    Therefore \eqref{eqn:bandit-inproof-2} becomes
    \begin{equation}\label{eqn:bandit-inproof-3}
    \E\left[ \sum_{t=1}^T \left(  f_t( x_t)-\alpha  f_t(x^*) \right) \right]
    \le \frac{\alpha^2\beta^2d}{2\eta} + \frac{\eta L^2R^2d^2}{2\gamma}T.
    \end{equation}
    Next, by the guarantee of the PAD oracle, for any $t\in [T]$ we know that there exists $c_t \in (W')^\circ$ such that $\norm{\overline{x}_t + c_t - x_t}\le \epsilon$.
    It is easy to see that $\reals_+^d \subseteq W'$, which implies $(W')^\circ \subseteq \reals_+^d$, so we know $c_t \in \reals_+^d$. Then we have 
    \[
    \forall t\in[T]:\quad
    \E\left[ f_t(\overline x_t)-f_t(x_t) \right] = \E\left[ f_t^\top(\overline x_t-x_t) \right]
    \le \E\left[ f_t^\top(\overline x_t+c_t-x_t) \right]
    \le \epsilon L .
    \]
	Thus \eqref{eqn:bandit-inproof-3} implies
	\begin{equation} \label{eqn:bandit-inproof-4}
	\E\left[\sum_{t=1}^T (  f_t(\overline x_t)-\alpha  f_t(x^*) ) \right]
	\le \frac{\alpha^2\beta^2d}{2\eta} + \frac{\eta L^2R^2d^2}{2\gamma}T + \epsilon LT.
	\end{equation}

	Finally, since the point played in round $t$, $z_t$, is equal to $\tilde x_t$ (whose expectation is $\overline x_t$) with probability $1-\gamma$, we have
	\begin{align*}
	\E\left[\sum_{t=1}^T (  f_t(z_t)-\alpha  f_t(x^*) ) \right]
	&\le (1-\gamma) \E\left[\sum_{t=1}^T (  f_t(\widetilde x_t)-\alpha  f_t(x^*) ) \right] + \gamma \cdot 2\alpha LRT \\
    &= (1-\gamma) \E\left[\sum_{t=1}^T (  f_t(\overline x_t)-\alpha  f_t(x^*) ) \right] + 2\gamma\alpha LRT \\
	&\le \frac{\alpha^2\beta^2d}{2\eta} + \frac{\eta L^2R^2d^2}{2\gamma}T + \epsilon LT +  2\gamma \alpha LRT.
	\end{align*}

	\paragraph{Oracle complexity.} Using Theorem~\ref{thm:projection}, we know that when $b_t = \textrm{EXPLORE}$, the number of calls to the oracle $\O_\K^\alpha$ in round $t$ is at most
	$\left\lceil 5d \log\frac{4\alpha R+2\sqrt{\frac2\mu \min_{x^* \in \K}D_\varphi(\alpha x^*, y_t)}}{\epsilon} \right\rceil$, where $\mu$ is the strong convexity parameter of $\varphi$.
    
    In the above proof of the regret bound, we have ignored the term $D_\varphi(\alpha x^*, y_{T+1})$ in \eqref{eqn:bandit-inproof-1}.
    If we instead keep this term, the regret bound \eqref{eqn:bandit-inproof-4} will become
    \begin{equation*}
    \forall x^*\in\K:\quad
    \E\left[\sum_{t=1}^T (  f_t(\overline x_t)-\alpha  f_t(x^*) ) \right]
    \le \frac{\alpha^2\beta^2d}{2\eta} + \frac{\eta L^2R^2d^2}{2\gamma}T + \epsilon LT - \frac1\eta \E\left[ D_\varphi(\alpha x^*, y_{T+1}) \right].
    \end{equation*}
    In the above inequality, substituting $T$ with $t$, we have
    \begin{align*}
    \forall t\in[T], \forall x^*\in\K:\quad \E\left[ D_\varphi(\alpha x^*, y_{t+1}) \right] &\le \frac{\alpha^2\beta^2d}{2} + \frac{\eta^2 L^2R^2d^2}{2\gamma}t + \epsilon\eta Lt - \eta \E\left[\sum_{j=1}^t (  f_j(\overline x_j)-\alpha  f_j(x^*) ) \right] \\
    &\le \frac{\alpha^2\beta^2d}{2} + \frac{\eta^2 L^2R^2d^2}{2\gamma}T + \epsilon\eta LT + \eta \cdot 2\alpha LRT \\
     &\le \frac{\alpha^2\beta^2d}{2} + \frac{\eta^2 L^2R^2d^2}{2\gamma}T + 3\eta\alpha LRT.
    \end{align*}
    The above upper bound is also clearly valid for $D_\varphi(\alpha x^*, y_1) $.

    Since $\varphi(x) = \frac12 x^\top Q^{-1}x$ is quadratic, we know that $\mu = \lambda_{\min}(Q^{-1}) = \frac{1}{\lambda_{\max}(Q)} = \frac{1}{\max_{u\in\reals^d, \norm{u}=1}\norm{Qu}} = \frac{1}{\max_{u\in\reals^d, \norm{u}=1}\norm{\sum_{i=1}^d q_iq_i^\top u}} \ge \frac{1}{\sum_{i=1}^d\norm{q_i}^2} \ge \frac{1}{R^2d}$,
	where $\lambda_{\min}(P)$ and $\lambda_{\max}(P)$ are respectively the smallest and the largest eigenvalues of a symmetric matrix $P$.
    
    Note that $\log(a+\sqrt{x})$ is a concave function in $x$ for $a>0$. By Jensen's inequality, the expected number of calls to the oracle $\O^{\alpha}_\K$ in round $t$ when $b_t = \textrm{EXPLORE}$ is upper bounded by:
    \begin{align*}
    &\E\left[1+5d \log\frac{4\alpha R+2\sqrt{\frac2\mu \min\limits_{x^* \in \K}D_\varphi(\alpha x^*, y_t)}}{\epsilon}\right] \\
    \le\ &1+\min_{x^* \in \K}\E\left[5d \log\frac{4\alpha R+2\sqrt{\frac2\mu D_\varphi(\alpha x^*, y_t)}}{\epsilon}\right]\\
    \le\ &1+\min_{x^* \in \K}5d \log\frac{4\alpha R+2\sqrt{\frac2\mu \E\left[D_\varphi(\alpha x^*, y_t)\right]}}{\epsilon}\\
    \le\ &1+ 5d \log\frac{4\alpha R+2 \sqrt{2R^2d\left(\frac{\alpha^2\beta^2d}{2} + \frac{\eta^2L^2R^2d^2}{2\gamma}T+3\eta\alpha LRT\right)}}{\epsilon}\\
    =\ &1+ 5d \log\frac{4\alpha R+2R \sqrt{\alpha^2\beta^2d^2 + \frac{\eta^2L^2R^2d^3}{\gamma}T+6\eta\alpha LRdT}}{\epsilon}.
    \end{align*}
%     in  where $A$ is an upper bound on $\max_{x^* \in \K} \varphi(\alpha x^*)$ and $\mu$ is the strong convexity parameter of $\varphi$.
% 	From \eqref{eqn:bandit-inproof-3} we know $A = \frac{\alpha^2\beta^2d}{2}$ suffices.
	Therefore the expected total number of calls to the oracle $\O_\K^\alpha$ in $T$ rounds is at most
	\begin{align*}
	 (1+\gamma T)\left( 1+ 5d \log\frac{4\alpha R+2R \sqrt{\alpha^2\beta^2d^2 + \frac{\eta^2L^2R^2d^3}{\gamma}T+6\eta\alpha LRdT}}{\epsilon} \right).
	\end{align*}
	
	The second part of the theorem can be directly verified using the specific choices of $\eta$, $\epsilon$ and $\gamma$ and noting $\log(\poly(\beta d)) = O\left( \log T \right)$ since $T>\beta^2d^3$.
\end{proof}
%\noindent
%\textbf{1)} $\sup_{x^*\in \K} \R(\alpha x^*)-\R(y_1)\leq \frac{d\alpha^2\beta^2}{2}\leq d\alpha^2\beta^2$.
%
%%Note that $Q=M M^\top$, where $M=[q_1\  q_2\ \ldots\ q_d]$. In other words, $Me_i=q_i$, $M^{-1}q_i = e_i$. By algorithm, $\R(y_1) = \min_{y_1\in\reals^d} \R(y)=0$. Since $\{q_i\}_{i=1}^d$ is $\beta$-BS($\K$), $\exists \beta_i\in[-\beta,\beta]$, $x^*=\sum_{i=1}^d\beta_iq_i$. Thus,
%
%
%\noindent
%\textbf{2)} $\E[\widetilde{f_t}^\top Q \widetilde{f_t}]\!=\! \frac{d}{\gamma}\sum\limits_{i=1}^d q_{i}^\top Q^{-1} Q Q^{-1} q_{i}l^2_t\leq \frac{dR^2L^2}{\gamma}\sum\limits_{i=1}^d q_{i}^\top Q^{-1} q_{i}\!=\! \frac{dR^2L^2}{\gamma} \sum\limits_{i=1}^d e_i^\top e_i =\frac{d^2R^2L^2}{\gamma}.$

%	\[\begin{split} \sum_{t=1}^T \widetilde{f}_t^\top (x_t-y^*)
%	\leq &\frac{1}{\eta}\left( \R(\alpha x^*)-\R(y_1)+\sum_{t=1}^T D_{\R^*}(\nabla \R(x_t)-\eta \widetilde{f_t},\nabla \R(x_t))\right)+ TL\epsilon\\
%	\leq & \frac{\alpha^2\beta^2d}{\eta}+\sum_{t=1}^T\eta\widetilde{f_t}^\top Q\widetilde{f_t}+TL\epsilon.\end{split}\]
	
	%Thus,  note that $\E[\tilde{f_t}|x_t] = f_t$

\section{Conclusion and Open Problems}

We have described two different algorithmic approaches to reducing regret minimization to
 offline approximation algorithms and maintaining optimal regret and poly-logarithmic oracle complexity per iteration, resolving previously stated open questions. 

An intriguing open problem remaining is to find an efficient algorithm in the bandit setting that guarantees both $\tilde{O}(\sqrt{T})$ regret and $\poly(\log T)$ oracle complexity per iteration (at least on average).  % It is possible that the barycentric regularization in OMD is useful for this purpose. 

%\section*{Acknowledgments}

%\wei{to do}

% Acknowledgments---Will not appear in anonymized version

\bibliography{references}
\bibliographystyle{apalike}

\appendix

\section*{\LARGE Appendix}

%\section{Omitted Proofs}

%\section{Omitted Proof in Section~\ref{sec:preliminary}} 
\section{Proof of Lemma~\ref{lem:center_cut}}
\label{app:proof-prelim}

\begin{proof}[Proof of Lemma~\ref{lem:center_cut}]
	Let $H = \{ x \in \mathbb{R}^d:  w^\top x  \ge b \}$ for a unit vector $w\in\reals^d$ and $b\in \reals$, and assume without loss of generality that $x^*  = 0$. Consider the one-dimensional random variable $Y :=  w^\top X  - b$, where $X\sim p$. Denote by $q: \reals\to\reals$ the density function of $y$. Then we have
	\begin{align*}
	\int_{ H} p(x) \dx = \int_{0}^\infty q(y) \dy .
	\end{align*}
	
	Let $y^* := \E[Y] = w^\top x^* - b = -b$. By our assumption, we know  $|y^*| \le \frac{1}{2e} $. Moreover, since log-concavity is preserved under linear transformations \citep{prekopa73}, we know that $y$ also follows a log-concave distribution, and it is easy to see that it is also isotropic.
	Using Lemma 5.4 in \citep{lovasz2007geometry}, we know $\int_{y^*}^\infty q(y) \dy \ge \frac{1}{e}$.
	In addition, from Lemma 5.5 in \citep{lovasz2007geometry} we know $q(y) \le 1$ ($\forall y\in\reals$).
	Therefore, we have
	\begin{align*}
	\frac1e - \int_{0}^\infty q(y) \dy \le	\int_{y^*}^\infty q(y) \dy - \int_{0}^\infty q(y) \dy \le |y^*| \sup_{y \in \mathbb{R}}q(y) \le \frac{1}{2e},
	\end{align*}
	%Since $x$ comes from an isotropic log-concave distribution, we know that $q(y)$ is also an (one-dimensional) isotropic log-concave distribution. Thus, using a result of (Lemma 5.5 in \citep{lovasz2007geometry}), we know that $q(y) \le 1$.
	which implies $\int_{0}^\infty q(y) \dy \ge  \frac{1}{2e}$, completing the proof.
\end{proof}

%
%=================================================================================
%

%\section{Omitted Proof in Section~\ref{sec:OMD}} 
\section{Proof of Lemma~\ref{lem:diff_in_dual}}
\label{app:proof-omd}

\begin{proof}[Proof of Lemma~\ref{lem:diff_in_dual}]
		Since $W$ is a convex cone and $\Pi_W(x)\in W$, we have $w+\Pi_W(x)\in W$ ($\forall w\in W$).
		By Lemma~\ref{lem:pythagorean}, we have $(\Pi_W(x)-x)^\top(y-\Pi_W(x))\ge 0$ ($\forall y\in W$). Letting $y=w+\Pi_W(x) \in W$ ($\forall w\in W$), we get $(\Pi_W(x)-x)^\top w\ge 0$, which means $\Pi_W(x)-x\in W^\circ$.
\end{proof}

%
%=================================================================================
%

%\section{Omitted Proof in Section~\ref{sec:cmwu}} \label{app:proof-cmwu}
\section{Proof of Theorem~\ref{thm:SOD}}\label{app:proof-cmwu}

%\subsection{Proof of Theorem~\ref{thm:SOD}}

We first show the following lemma 
using a similar argument in Lemma~\ref{lem:cutting_plane_OMD}.

\begin{lemma}\label{lem:cutting_plane}
	For $x \in B(0, R)$ and $\epsilon \in (0, 2R]$,
	if  $\mathcal{SOD}(x,\epsilon,W)$ returns a decomposition $(V, p)$, then for all unit vector $w\in W$, we have $\min_{i\in [k]} w^\top (v_i-x)\le 3\epsilon$. %\[\forall w\in W\cap B(0,1),\ \exists i\in[k], \textrm{ s.t. } w^T(v_i-x)\le 3\epsilon.\]
\end{lemma}
\begin{proof}%{\lemmaref{lem:cutting_plane}}
	Suppose that there exists a unit vector $h\in W$ such that $\min_{i\in [k]} h^\top (v_i-x)> 3\epsilon$. Note that $\|v_i-x\|\le \|v_i\|+\|x\|\le 2R$. Denoting $r = \frac{\epsilon}{4R}$, we have 
	\[\forall h'\in \frac{h}{2}+(W\cap B(0,r)):\quad \min_{i\in [k]}h'^\top (v_i-x)> \epsilon.\]
	
	Since $r\le \frac{1}{2}$ for $\epsilon\le 2R$, we have $\frac{h}{2}+(W\cap B(0,r))\subseteq \frac{h}{2}+(W\cap B(0,1/2)) \subseteq W\cap B(0,1)=W_1$. Because the algorithm returns a decomposition, we have that after the last iteration,
	\[\forall w\in W_1\setminus W_{k+1}:\quad \exists i\in[k], \textrm{ s.t. } w^\top(v_i-x)\le \epsilon.\]

	Therefore, we must have $\frac{h}{2}+(W\cap B(0,r))\subseteq W_{k+1}$. 
	We also have $\vol{W_{i+1}} \le (1-1/2e) \vol{W_i}$ from Lemma~\ref{lem:center-cut-to-use} since $W_{i+1}$ is the intersection of $W_i$ with a half-space that does not contain $W_i$'s centroid.
	Then we have
	\begin{align*}
	\vol{W_1} &=\vol{W\cap B(0,1)}= r^{-d}\vol{W\cap B(0,r)}\le r^{-d}\vol{W_{k+1}}\\
	&\le r^{-d}(1-1/2e)^{k}\vol{W_1}< \vol{W_1},
	\end{align*}
	where the last step is due to $k \ge 5d \log  \frac1r = 5d \log  \frac{4R}{\epsilon}$. Therefore we have a contradiction.
\end{proof}

\begin{proof}[Proof of Theorem~\ref{thm:SOD}]
	If $\mathcal{SOD}(x,\epsilon,W)$ returns a decomposition $(V, p)$, by Lemmas~\ref{lem:cutting_plane} and \ref{lem:equivalence}, we know that there exists $ c\in W^\circ$ such that $\norm{\sum_{i=1}^kp_iv_i+c-x} \le 3\epsilon$. %In other words, $\exists c\in W^\circ$, $\sum_{i=1}^k p_iv_i-x\in -c+B(0,3\epsilon)$, which is equivalent to $\sum_{i=1}^k p_iv_i-x\in -W^\circ+B(0,3\epsilon)$.
	
	If $\mathcal{SOD}(x,\epsilon,W)$ returns a separating hyperplane $(w, b)$ at iteration $i\in[k]$, we know $w_i^\top x\le w_i^\top v_i -\epsilon$. 
	Since $w = \frac{w_i}{\norm{w_i}}$, $b = w^\top v_i$ and $\norm{w_i}\le1$, we have $w^\top x \le w^\top v_i - \frac{\epsilon}{\norm{w_i}} \le b - \epsilon$.
	By the guarantee of $\mathcal{O}_{\K, \K^*}$, we have $b-\epsilon = w^\top v_i - \epsilon \le \min_{x^*\in K^*}w^\top x^*-\epsilon$.
	
	The number of calls of $\O_{\K, \K^*}$ is clearly upper bounded by $k= \left\lceil 5d\log \frac{4R}{\epsilon} \right\rceil$ since there are at most $k$ iterations and each iteration only calls $\O_{\K, \K^*}$ once.
\end{proof}

\end{document}